\documentclass{article}

\usepackage[main, final]{neurips_2026}

\usepackage{amsmath,amsfonts,bm}

\def\eqref#1{equation~\ref{#1}}

\def\1{\bm{1}}

\def\ve{{\bm{e}}}

\def\vp{{\bm{p}}}

\def\vu{{\bm{u}}}

\def\mA{{\bm{A}}}
\def\mB{{\bm{B}}}

\def\mE{{\bm{E}}}

\def\mG{{\bm{G}}}

\def\mI{{\bm{I}}}

\def\mM{{\bm{M}}}

\def\mP{{\bm{P}}}
\def\mQ{{\bm{Q}}}
\def\mR{{\bm{R}}}
\def\mS{{\bm{S}}}

\def\mU{{\bm{U}}}
\def\mV{{\bm{V}}}
\def\mW{{\bm{W}}}
\def\mX{{\bm{X}}}
\def\mY{{\bm{Y}}}
\def\mZ{{\bm{Z}}}

\DeclareMathAlphabet{\mathsfit}{\encodingdefault}{\sfdefault}{m}{sl}
\SetMathAlphabet{\mathsfit}{bold}{\encodingdefault}{\sfdefault}{bx}{n}

\def\gC{{\mathcal{C}}}

\def\gL{{\mathcal{L}}}

\def\gR{{\mathcal{R}}}

\DeclareMathOperator{\Tr}{Tr}

\usepackage{hyperref}
\usepackage{url}
\usepackage[ruled,vlined]{algorithm2e}

\usepackage{amsmath}
\usepackage{amsthm}
\usepackage{mathtools}
\usepackage{enumitem}
\setlist[itemize]{noitemsep,nolistsep}
\usepackage[utf8]{inputenc} 
\usepackage[T1]{fontenc}    
\usepackage{booktabs}       
\usepackage{amsfonts}       
\usepackage{nicefrac}       
\usepackage{microtype}      
\usepackage{amssymb}
\usepackage{multirow}
\usepackage{multicol}
\usepackage{bm}
\usepackage[flushleft]{threeparttable}
\usepackage{setspace}
\usepackage{pifont}
\usepackage{subcaption}
\usepackage{graphicx}
\usepackage{xcolor}
\usepackage{caption}
\usepackage{wrapfig}
\usepackage{lipsum}
\usepackage{arydshln}
\usepackage{bbm}
\usepackage{footmisc}
\usepackage{tcolorbox}
\usepackage{hhline}
\usepackage{makecell}
\usepackage{mwe}
\usepackage{tabularx, array}
\usepackage[table]{xcolor}
\definecolor{ForestGreen}{RGB}{34,139,34}

\usepackage[capitalize,noabbrev]{cleveref}

\usepackage{tikz}

\usepackage{stackengine}

\theoremstyle{plain}
\newtheorem{theorem}{Theorem}[section]

\newtheorem{lemma}[theorem]{Lemma}

\theoremstyle{definition}
\newtheorem{definition}[theorem]{Definition}
\newtheorem{assumption}[theorem]{Assumption}
\theoremstyle{remark}

\newcommand{\methodname}{SDS-LoRA}

\title{\methodname: Overcoming Anisotropic Gradient Scaling in Low-Rank Adaptation}

\author{%
  Junghun Oh\thanks{This work was done at Seoul National University.}\\
  Dept. of ECE, ASRI\\
  Seoul National University\\
  Nota Inc.\\
  \texttt{dh6dh@snu.ac.kr} \\
  \And
  Sungyong Baik\thanks{Corresponding author.} \\
  Dept. of Artificial Intelligence\\
  Dept. of Data Science\\
  Hanyang University\\
  \texttt{dsybaik@hanyang.ac.kr} \\
  \And
  Kyoung Mu Lee \\
  Dept. of ECE, ASRI, IPAI\\
  Seoul National University\\
  \texttt{kyoungmu@snu.ac.kr} \\
}

\begin{document}

\maketitle

\begin{abstract}
Low-Rank Adaptation (LoRA) enables efficient adaptation of large pretrained models to downstream tasks by parameterizing weight updates with low-rank matrices.
In this paper, we investigate the limitations of the LoRA parameterization from a geometric perspective.
Specifically, we show that when a full fine-tuning gradient is backpropagated to the low-rank matrices, it undergoes anisotropic scaling driven by their singular values.
We argue that this phenomenon is undesirable because it distorts the full fine-tuning gradient by skewing it toward dominant singular directions while suppressing others.
Our analyses demonstrate that anisotropic gradient scaling reduces the effective rank of the low-rank matrices' gradients and fails to provide the best possible alignment between the full fine-tuning gradient and its low-rank approximation in LoRA for arbitrary gradients, thereby exacerbating the gap to full fine-tuning.
To address these limitations, we propose a new low-rank parameterization, SDS-LoRA, which \textbf{S}tructurally \textbf{D}ecouples \textbf{S}ingular values from the backward pass.
Our method ensures that the full fine-tuning gradient backpropagates only through the orthonormal bases of the low-rank matrices' subspaces, independent of their scales.
Convergence analysis demonstrates that while LoRA's convergence rate degrades with the condition number of the low-rank matrices, that of SDS-LoRA remains independent of it.
Experimental results across natural language and vision benchmarks show that SDS-LoRA improves loss convergence and reduces the gap to full fine-tuning, significantly enhancing adaptation performance.
The code is available at \href{https://github.com/JungHunOh/SDS-LoRA_NeurIPS2026}{https://github.com/JungHunOh/SDS-LoRA\_NeurIPS2026}.
\end{abstract}

\section{Introduction}
Large-scale pretrained Transformer models \citep{Vaswani2017attention}, such as GPT~\citep{radford2018improving}, LLaMA~\citep{touvron2023llama}, and ViTs~\citep{Dosovitskiy2021vit}, have achieved remarkable success by learning rich representations from massive datasets.
In recent years, adapting these versatile models to a wide range of downstream tasks has gained immense popularity.
A few studies have explored improving task adaptation through adaptive optimization strategies~\citep{baik2020meta,baik2024learning} and also investigated fine-tuning strategies to mitigate parameter interference across task-specific models~\citep{lee2025mitigating}.
While this pretraining-and-adaptation paradigm has become the standard in the field, full fine-tuning of these models is often computationally expensive and impractical.
For example, GPT-3~\citep{brown2020gpt3} has approximately 175 billion parameters, and LLaMA-3 variants~\citep{grattafiori2024llama3} range from 8 to 405 billion parameters.
To address this, Low-Rank Adaptation (LoRA) \citep{hu2022lora} leverages the observation that model updates reside in a low-rank subspace \citep{Aghajanyan2021intrinsic,Li2018measuring}, representing weight updates through trainable low-rank matrices: $\Delta \mW = \mB\mA$.
Consequently, LoRA drastically reduces the number of trainable parameters.

\begin{figure*}[t!]
    \centering
    \begin{subfigure}[b]{0.3\textwidth}
        \centering
        \includegraphics[width=\textwidth]{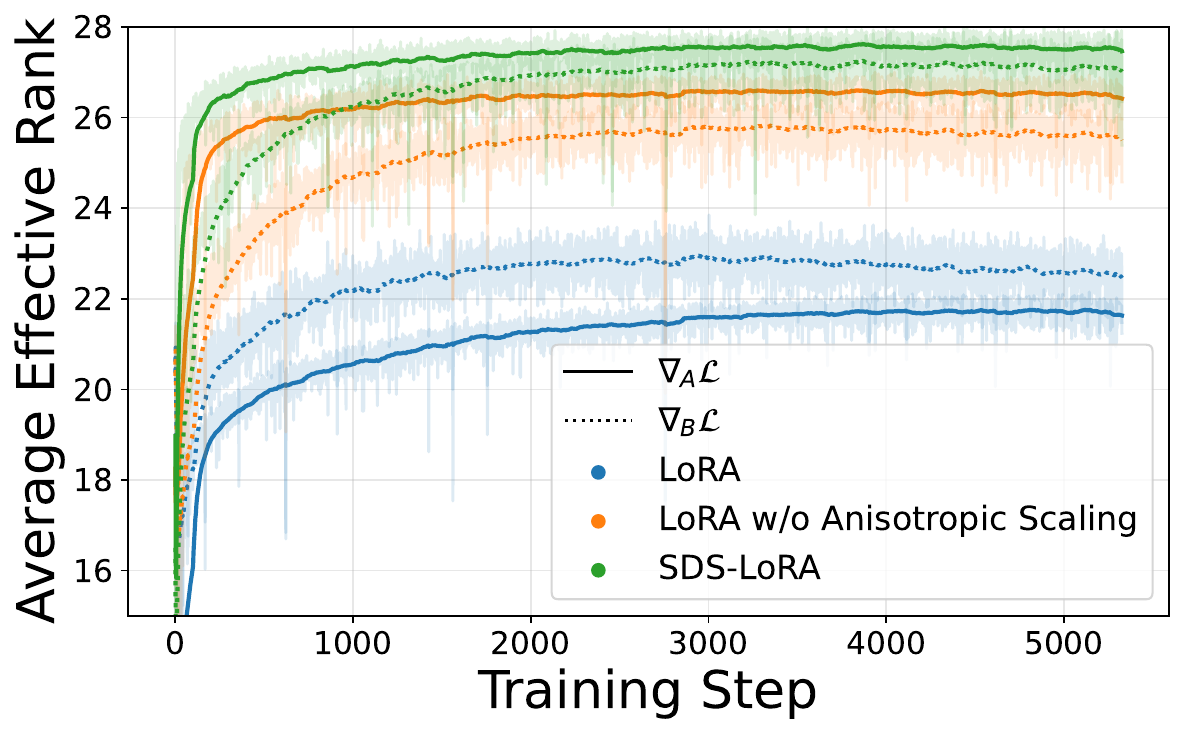}
        \vspace{-1.5em}
        \caption{Effective rank of gradients}
        \label{fig:motivation_effrank}
    \end{subfigure}
    \hspace{0.3em}
    \begin{subfigure}[b]{0.6\textwidth}
        \centering
        \includegraphics[width=\textwidth]{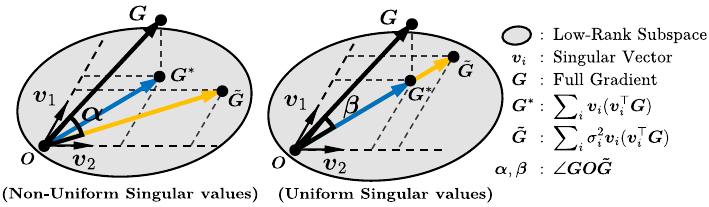}
        \vspace{-1.5em}
        \caption{Alignment between the full gradient $\mG$ and its approximation $\Tilde{\mG}$}
        \label{fig:motivation_alignment}
    \end{subfigure}
    \vspace{-0.5em}
    \caption{\textbf{Adverse effects of anisotropic gradient scaling in LoRA.}
    \textbf{(a)} At each step of LoRA training, we compare the effective rank of the original gradients of the LoRA matrices (\textcolor{blue}{blue}) to the modified gradients where anisotropic gradient scaling induced by singular values has been removed (\textcolor{orange}{orange}).
    We observe a substantial reduction in effective rank when gradients are scaled by singular values.
    In contrast, \methodname~overcomes anisotropic gradient scaling, thereby maintaining a higher effective gradient rank (\textcolor{ForestGreen}{green}).
    \textbf{(b)} Non-uniform singular values of LoRA matrices cause the effective gradient with respect to $\Delta \mW$ ($\Tilde{\mG}$) to misalign with the orthogonal projection of $\mG$ onto the low-rank subspace ($\mG^*$) by applying anisotropic scaling to $\mG^*$.
    This results in suboptimal alignment between $\mG$ and $\Tilde{\mG}$ ($\alpha > \beta$).
    }
    \label{fig:motivation}
    \vspace{-3em}
\end{figure*}

Although LoRA has attracted significant attention for its efficiency, a substantial performance gap remains relative to full fine-tuning, motivating further improvements to the LoRA framework.
In this paper, we investigate the limitations of the LoRA parameterization from a geometric perspective.
Specifically, we note that the optimization dynamics in LoRA are adversely affected by the singular-value structure of the LoRA matrices (i.e., $\mA$ and $\mB$).
For example, when calculating the gradient with respect to $\mA$, the gradient with respect to $\Delta \mW$, which we refer to as the full (fine-tuning) gradient, is backpropagated through $\mB$.
In this process, the singular values of $\mB$ strengthen directions associated with large singular values while relatively diminishing others.
This phenomenon may be undesirable because it distorts the full gradient by biasing it toward dominant singular directions, thereby concentrating its rich information on only a few directions.
Figure~\ref{fig:motivation_effrank} demonstrates that the effective rank of the LoRA matrices' gradients is significantly reduced due to the anisotropic scaling, potentially limiting LoRA's learning capacity.
Furthermore, we theoretically show that uniform singular values uniquely maximize the worst-case quality with which the effective gradient with respect to $\Delta \mW$ approximates the full gradient, whereas anisotropic gradient scaling degrades it (Figure~\ref{fig:motivation_alignment}), causing the optimization dynamics of LoRA to deviate from those of full fine-tuning.

\begin{wrapfigure}{r}{0.37\textwidth}
    \begin{center}
    \vspace{-0.5cm}
    \includegraphics[width=1\linewidth]{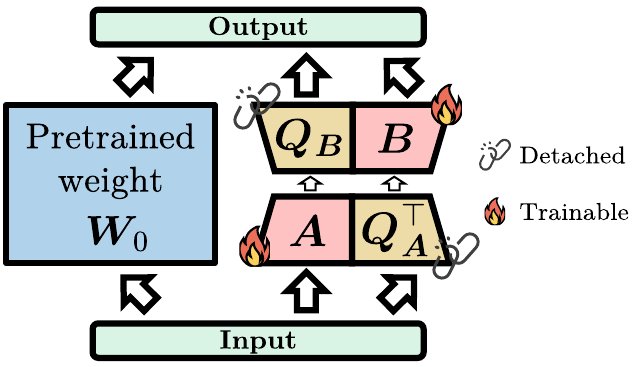}
    \end{center}
    \vspace{-0.3cm}
    \caption{\textbf{Illustration of \methodname}
    }
    \vspace{-0.5cm}
    \label{fig:method_illustration}
\end{wrapfigure}
To overcome the adverse effects of anisotropic gradient scaling, we propose a new parameterization for $\Delta \mW$, named SDS-LoRA, that \textbf{S}tructurally \textbf{D}ecouples the \textbf{S}ingular values of the LoRA matrices from the backward pass.
As illustrated in Figure~\ref{fig:method_illustration}, \methodname~multiplies each LoRA matrix by a matrix representing the orthonormal bases of the other's subspace (i.e., $\mQ_\mA$ and $\mQ_\mB$) and sums the two products.
By treating $\mQ_\mA$ and $\mQ_\mB$ as constants in the backward pass, this formulation ensures that the singular values of the LoRA matrices do not participate in the gradient computation, thereby eliminating anisotropic gradient scaling.
Consequently, as shown in Figure~\ref{fig:motivation_effrank}, our method results in a significantly higher effective rank for the gradients of $\mA$ and $\mB$ compared to LoRA.
Moreover, as shown by the uniform-singular-value case in Figure~\ref{fig:motivation_alignment}, our method yields an effective gradient with respect to $\Delta \mW$ that nearly attains this worst-case optimum, thereby reducing the gap to full fine-tuning.
Our theoretical analysis proves that, while the convergence rate of LoRA is degraded by the condition number of the LoRA matrices, that of our method remains independent of it.
Experiments across various benchmarks and models demonstrate that \methodname~significantly improves loss convergence and narrows the gap to full fine-tuning, outperforming existing methods.


We summarize the contribution of our work as follows:
\begin{itemize}[itemsep=1pt]
    \vspace{-0.5em}
    \item We provide a novel perspective on the limitations of LoRA, showing that LoRA suffers from anisotropic gradient scaling induced by the singular values of the LoRA matrices.
    \item We demonstrate that anisotropic gradient scaling reduces the effective rank of the LoRA matrices' gradients and does not attain the best approximation of arbitrary full fine-tuning gradients.
    \item We propose a new parameterization, \methodname, that overcomes anisotropic gradient scaling in LoRA by structurally decoupling the singular values from the backward pass.
    \item \methodname~substantially improves loss convergence and narrows the gap to full fine-tuning, outperforming LoRA and existing methods.
    \vspace{-0.5em}
\end{itemize}

\section{Related Work}
\paragraph{LoRA Variants.}
Given the success of LoRA, numerous follow-up studies have been conducted to further enhance its performance~\citep{kalajdzievski2023rslora,Hayou2024loraplus,si2025unleashing}.
Several works propose alternative formulations of low-rank weight updates to further improve the performance–efficiency trade-off~\citep{liu2024dora,Kopiczko2024vera,Lingam2024svft,Albert2025randlora,Koohpayegani2024nola}.
Other research directions include developing effective initialization strategies for the trainable low-rank matrices~\citep{Hayou2024impact,wang2024loraga,meng2024pissa,Zhang2025loraone,Li2025beyond}, improving the rank of weight updates~\citep{Huang2025hira,Lialin2024relora}, and designing layer-wise adaptive rank search methods~\citep{Zhang2023adalora,ke2024unveiling}.
Another line of work leverages orthogonalization: OLoRA~\citep{buyukakyuz2024olora} initializes the LoRA matrices from a QR decomposition of the pretrained weight, and QR-LoRA~\citep{yang2025qrlora} freezes an orthonormal factor and trains only the triangular one.
The former restores isotropy only at initialization, while the latter confines the row space of $\Delta \mW$ to a subspace determined by the pretrained weight, sacrificing representational capacity.
See Section~\ref{sec:orthogonal_lora} for a mechanism-level and empirical comparison.

\vspace{-0.5em}
\paragraph{Full Fine-Tuning Approximation Perspective.}
Several recent studies have examined LoRA from the perspective of approximating the full fine-tuning gradient~\citep{wang2024loraga,Zhang2025loraone,wang2025lorapro,riemannian2024zhang,altlora2025yu,tastan2026loft}.
Wang et al.~\citep{wang2024loraga} and Zhang et al.~\citep{Zhang2025loraone} propose initialization strategies for LoRA that minimize the gradient approximation error at the first training step.
Several recent works attempt to minimize the gap between the full fine-tuning gradient and the effective gradient with respect to weight updates via preconditioning~\citep{wang2025lorapro,riemannian2024zhang,altlora2025yu,tastan2026loft,yen2025lora,zhang2025reflora}.
Specifically, they scale the gradient of each LoRA matrix by the inverse Gram matrix of the other LoRA matrix.
While these approaches eliminate the influence of singular values on the effective gradient, anisotropic gradient scaling persists within the LoRA matrices' gradients.
Moreover, because these methods modify the original gradients, they require additional mechanisms to adjust momentum, which hinders their practical adoption.
See Section~\ref{sec:precond_discussion} for a detailed comparison to the preconditioning-based methods.

\vspace{-0.5em}
\paragraph{Learning in Stiefel Manifolds.}
Li et al.~\citep{li2025stella} and Lion et al.~\citep{lion2025polar} formulate weight updates as $\Delta \mW = \mU\mS\mV^\top$, where $\mU$ and $\mV$ are constrained to Stiefel manifolds and $\mS$ is a scale matrix.
These methods also modify the Euclidean gradient to produce a Riemannian gradient and require additional retraction or regularization to ensure that $\mU$ and $\mV$ remain on the manifolds.
Importantly, the updates for $\mU$ and $\mV$ depend on the scale matrix $\mS$, failing to overcome anisotropic gradient scaling. 
Related studies have also employed SVD-based low-rank subspace representations for task identification and routing~\citep{jung2026recover, son2026trainingfree}.

Unlike prior works, we focus on the adverse effects of anisotropic gradient scaling, offering a novel perspective on LoRA's limitations.
Our method overcomes these effects by structurally eliminating the scaling effects from the backward pass without modifying the original gradients.

\section{Proposed Method}
\subsection{Preliminaries}

Let $\mW_0 \in \mathbb{R}^{d_\text{out} \times d_\text{in}}$ denote the weight matrix of a linear layer in a pretrained model, with input dimension $d_\text{in}$ and output dimension $d_\text{out}$.
We consider fine-tuning $\mW_0$ to adapt the model to a downstream task. Instead of updating all the parameters of $\mW_0$, Low-Rank Adaptation (LoRA) assumes a low-rank structure for the desired weight update~\citep{Aghajanyan2021intrinsic,Li2018measuring} and learns a low-rank decomposition of the update:
\begin{equation}\label{eq:definition_of_lora}
    \mW_\text{eff} = \mW_0 + \Delta \mW_\text{LoRA} = \mW_0 + s \mB \mA,
\end{equation}
where $\mA \in \mathbb{R}^{r \times d_\text{in}}$ and $\mB \in \mathbb{R}^{d_\text{out} \times r}$ are trainable LoRA matrices of rank $r$, and $s$ is a constant scaling factor. For a single layer, the total number of trainable parameters in LoRA is $r \times (d_\text{out} + d_\text{in})$.
Assuming $r \ll \min(d_\text{out}, d_\text{in})$, this is significantly smaller than the $d_\text{out} \times d_\text{in}$ parameters required in a full fine-tuning scenario.

Let $\mG = \nabla_{\mW_\text{eff}} \gL$ denote the gradient of the loss $\gL$ with respect to the effective weights.
We refer to $\mG$ as the \textit{full gradient}, as it is used in full fine-tuning.
By the chain rule, the gradients with respect to the LoRA matrices are given by $\nabla_\mA \gL = s\mB^\top \mG$, $\nabla_\mB \gL = s\mG \mA^\top$.
Let $\mA = \mU_\mA \boldsymbol{\Sigma}_\mA \mV_\mA^\top$ and $\mB = \mU_\mB \boldsymbol{\Sigma}_\mB \mV_\mB^\top$ be the rank-$r$ truncated singular value decompositions (SVD) of the LoRA matrices, where $\mV_\mA \in \mathbb{R}^{d_\text{in} \times r}$ and $\mU_\mB \in \mathbb{R}^{d_\text{out} \times r}$ have orthonormal columns, $\boldsymbol{\Sigma}_\mA, \boldsymbol{\Sigma}_\mB \in \mathbb{R}^{r \times r}$ are the diagonal matrices of singular values, and $\mU_\mA, \mV_\mB \in \mathbb{R}^{r \times r}$ are orthogonal matrices.
Substituting the SVD representations into the gradient equations, we obtain
\begin{equation}\label{eq:grad_svd}
    \nabla_\mA \gL = s\mV_\mB \boldsymbol{\Sigma}_\mB \mU_\mB^\top \mG, \:\: \nabla_\mB \gL = s \mG \mV_\mA \boldsymbol{\Sigma}_\mA \mU_\mA^\top\:\:\:\:\:\text{(in LoRA)}.
\end{equation}

\subsection{Adverse Effects of Anisotropic Gradient Scaling via Singular Values}\label{sec:lora_limitation}
We provide a geometric interpretation of how the full gradient $\mG$ is backpropagated to yield $\nabla_\mA \gL$ and $\nabla_\mB \gL$.
Equation~\ref{eq:grad_svd} shows that $\mG$ undergoes the following transformations:
\begin{itemize}[itemsep=2pt]
    \item $\mV_\mA$ and $\mU_\mB^\top$: Acting as partial isometries, these matrices perform isometric projections of $\mG$ into their respective subspaces.
    \item $\boldsymbol{\Sigma}_\mA$ and $\boldsymbol{\Sigma}_\mB$: If the singular values of $\mA$ and $\mB$ are non-uniform, these matrices induce anisotropic scaling that skews the projected gradient toward the dominant singular directions.
    \item $\mU_\mA^\top$ and $\mV_\mB$: These matrices perform isometric transformations.
\end{itemize}
Note that the isometric projection by $\mV_\mA$ and $\mU_\mB^\top$ is an inherent bottleneck of low-rank adaptation, and $\mU_\mA^\top$ and $\mV_\mB$ act merely as isometries, preserving the geometric properties of the projected gradient.
\textbf{However, the anisotropic scaling through $\boldsymbol{\Sigma}_\mA$ and $\boldsymbol{\Sigma}_\mB$ may be undesirable because it makes the backpropagated gradient depend on which directions the spectrum happens to emphasize, leaving the remaining directions with an arbitrarily small share of the full gradient in the worst case.}

As discussed in Section~\ref{sec:stable_rank}, we observe that the stable ranks of $\mA$ and $\mB$ are substantially lower than the LoRA rank $r$, suggesting that their singular values are highly skewed.
To demonstrate the adverse effects of anisotropic gradient scaling resulting from these skewed singular values, we compare the effective rank of the original gradients (i.e., $s\mV_\mB\boldsymbol{\Sigma}_\mB \mU_\mB^\top\mG$ and $s\mG \mV_\mA \boldsymbol{\Sigma}_\mA \mU_\mA^\top$) with that of modified gradients that are independent of singular values (i.e., $s\mV_\mB \mU_\mB^\top\mG$ and $s\mG \mV_\mA  \mU_\mA^\top$).
Figure~\ref{fig:motivation_effrank} shows that the effective rank is significantly reduced when the gradients are scaled by the singular values.
These results suggest that anisotropic gradient scaling concentrates the rich information of $\mG$ on a few dominant directions, potentially limiting LoRA's learning capacity.

Furthermore, anisotropic gradient scaling causes the optimization dynamics of LoRA to markedly deviate from those of full fine-tuning.
To examine this, we first derive the effective gradient with respect to $\mW_\text{eff}$.
Under gradient descent with learning rate $\eta$, the change in $\mW_\text{eff}$ resulting from updates to $\mA$ and $\mB$ is given by $\Delta \mW_\text{eff} \approx s (\Delta \mB \mA + \mB \Delta \mA) = -\eta s^2 (\mG \mA^\top\mA+\mB\mB^\top\mG)$.
Substituting the SVD of $\mA$ and $\mB$ into this, the effective gradient with respect to $\mW_\text{eff}$ is expressed as
\begin{equation}\label{eq:eff_grad}
    \Tilde{\mG} = s^2(\mG \mV_\mA \boldsymbol{\Sigma}_\mA^2 \mV_\mA^\top + \mU_\mB \boldsymbol{\Sigma}_\mB^2 \mU_\mB^\top\mG) \:\:\:\:\:\text{(in LoRA)}.
\end{equation}
Consequently, LoRA can be interpreted as full fine-tuning with $\Tilde{\mG}$, a low-rank approximation of $\mG$~\citep{wang2024loraga,wang2025lorapro,riemannian2024zhang,altlora2025yu,tastan2026loft}.
We analyze how the gap between $\mG$ and $\Tilde{\mG}$ influences the optimization dynamics through the descent lemma~\citep{Nesterov:2018}.
Assuming that the loss function $\mathcal{L}$ is $\beta$-smooth with respect to $\mW_\text{eff}$, we obtain the following loss decrease bound when $\mW_\text{eff}$ is updated by $\Tilde{\mG}$:
\begin{equation}\label{eq:descent_lemma}
    \mathcal{L}(\mW_\text{eff}-\eta \Tilde{\mG}) - \mathcal{L}(\mW_\text{eff}) \leq -\eta\langle \mG, \Tilde{\mG} \rangle_F + \frac{\beta}{2} \eta^2 ||\Tilde{\mG}||^2_F,
\end{equation}
where $\langle \cdot, \cdot \rangle_F$ and  $||\cdot||_F$ denote the Frobenius inner product and norm, respectively.
Equation~\ref{eq:descent_lemma} suggests that, for a sufficiently small $\eta$, the loss decrease bound is degraded as $\Tilde{\mG}$ diverges from $\mG$ in terms of the Frobenius inner product.
As such, we consider $\langle \mG, \Tilde{\mG} \rangle_F$ as a measure of the approximation quality of $\Tilde{\mG}$.
Note that $\mA$ and $\mB$ persist throughout training and encounter a different full gradient $\mG$ at every step.
A desirable singular value configuration must therefore remain effective for arbitrary $\mG$ rather than being tuned to a single instance, which motivates a worst-case criterion over $\mG$.
To formalize this, define $\mM_\mA = \mV_\mA^\top \mG^\top \mG \mV_\mA$ and $\mM_\mB = \mU_\mB^\top \mG \mG^\top \mU_\mB$, so that the gradient approximation quality in LoRA is written as
\begin{equation*}
    \langle \mG, \Tilde{\mG} \rangle_F = s^2 J, \:\:\:\:\text{where}\:\:\:\: J = \Tr(\mM_\mA \boldsymbol{\Sigma}_\mA^2) + \Tr(\mM_\mB \boldsymbol{\Sigma}_\mB^2).
\end{equation*}
A worst-case value of $J$ alone, however, is not comparable across different $\mG$, since $J$ is quadratic in $\mG$ and an adversarial choice could simply take a $\mG$ of small norm.
We therefore compare $J$ against the best value achievable within the given subspaces for that same $\mG$, $J^\star(\mG) = E_\mA \lambda_{\max}(\mM_\mA) + E_\mB \lambda_{\max}(\mM_\mB)$ with $E_\mA = ||\mA||_F^2$ and $E_\mB = ||\mB||_F^2$, and evaluate a configuration by the scale-invariant worst-case ratio $\min_{\mG} J/J^\star(\mG)$, where the minimum is over $\mG$ with $J^\star(\mG) > 0$.
\vspace{-0.3em}
\begin{theorem}\label{theorem:approx}
Assume that the row space of $\mA$ and the column space of $\mB$ are given, together with the energies $E_\mA$ and $E_\mB$, so that the admissible configurations are those satisfying $\Tr(\boldsymbol{\Sigma}_\mA^2) = E_\mA$ and $\Tr(\boldsymbol{\Sigma}_\mB^2) = E_\mB$.
Then, $\boldsymbol{\Sigma}_\mA = \sqrt{E_\mA/r} \mI_r$ and $\boldsymbol{\Sigma}_\mB = \sqrt{E_\mB/r} \mI_r$ uniquely maximize the worst-case ratio $\min_{\mG} J/J^\star(\mG)$, attaining the optimal value $\frac{1}{r}$.
\end{theorem}
\vspace{-0.5em}
See Section~\ref{sec:approx_proof} for the proof.
Theorem~\ref{theorem:approx} suggests that, under the same subspace and energy configurations, uniform singular values are the unique configuration that guarantees the best gradient approximation quality for arbitrary $\mG$.
A non-uniform configuration, in contrast, is favorable only for the gradient directions that its spectrum happens to emphasize: for a $\mG$ concentrated on the smallest singular directions, its ratio degrades to $\frac{\lambda_{\min}(\boldsymbol{\Sigma}_\mA^2)+\lambda_{\min}(\boldsymbol{\Sigma}_\mB^2)}{E_\mA+E_\mB}$ (Equation~\ref{eq:approx_worstcase}), which falls strictly below $\frac{1}{r}$ and vanishes as the spectrum becomes more skewed.
Under the optimal condition, $\Tilde{\mG}$ is simplified to $s^2(\frac{E_\mA}{r} \mG \mP_{\gR(\mA)} + \frac{E_\mB}{r} \mP_{\gC(\mB)}\mG)$ where $\mP_{\gR(\mA)}$ and $\mP_{\gC(\mB)}$ denote the orthogonal projection matrices onto the row space of $\mA$ ($\gR(\mA)$) and the column space of $\mB$ ($\gC(\mB)$), respectively.
As illustrated in Figure~\ref{fig:motivation_alignment}, while the non-uniform singular values distort the projected gradients, the uniform singular values scale the projection equally in all directions, resulting in the best worst-case alignment between $\Tilde{\mG}$ and $\mG$.

\subsection{\methodname: Structural Decoupling of Singular Values from Backward Pass}\label{sec:method}

Our analysis in Section~\ref{sec:lora_limitation} demonstrates the negative effects of non-uniform singular values of $\mA$ and $\mB$ during the backward pass.
However, addressing this limitation is not straightforward, as the roles of singular values in the forward and backward passes conflict.
In the forward pass, the singular values play a crucial role in representing $\Delta \mW_\text{LoRA}$.
Enforcing uniform singular values of $\mA$ and $\mB$ significantly restricts the class of representable weight updates.\footnote{If both $\mA$ and $\mB$ have uniform singular values, $\Delta \mW_\text{LoRA}$ also has uniform singular values. If $\mA$ has uniform singular values, then $\mA = c\,\mU_\mA\mV_\mA^\top$ for some $c>0$, and similarly $\mB = d\,\mU_\mB\mV_\mB^\top$ for some $d>0$. Then $\Delta \mW_\text{LoRA} = \mB\mA = cd\,\mU_\mB(\mV_\mB^\top\mU_\mA)\mV_\mA^\top$. Since $\mV_\mB^\top\mU_\mA$ is a product of orthogonal matrices and hence orthogonal, both $\mU_\mB(\mV_\mB^\top\mU_\mA)$ and $\mV_\mA$ have orthonormal columns, so all singular values of $\Delta \mW_\text{LoRA}$ equal $cd$.}
Conversely, allowing fully expressive weight updates conflicts with the uniform singular value condition for $\mA$ and $\mB$.

To address this, we propose a new low-rank parameterization of weight updates, named \methodname, which \textbf{S}tructurally \textbf{D}ecouples the \textbf{S}ingular values of LoRA matrices from the backward pass while preserving their role in the forward pass for representing weight updates.
To do so, we utilize the orthonormal bases of $\gR(\mA)$ and $\gC(\mB)$.
Specifically, we perform QR decomposition of $\mA^\top$ and $\mB$: $\mA^\top = \mQ_\mA \mR_\mA$ and $\mB = \mQ_\mB \mR_\mB$ where the columns of $\mQ_\mA$ and $\mQ_\mB$ form orthonormal bases of $\gR(\mA)$ and $\gC(\mB)$, respectively.
We then formulate the weight updates as follows:
\begin{equation}\label{eq:method}
    \Delta \mW_{\text{\methodname}} = s(\mQ_\mB \mA + \mB \mQ_\mA^\top)=s[\mQ_\mB \; \mB] \begin{bmatrix} \mA\\ \mQ_\mA^\top \end{bmatrix}.
\end{equation}
Importantly, we treat $\mQ_\mA$ and $\mQ_\mB$ as constants during the backward pass.
This ensures that the LoRA matrices receive gradients only through $\mQ_\mA$ and $\mQ_\mB$, while also making training more efficient by eliminating the costly backpropagation through the QR decomposition.
Consequently, the gradients with respect to $\mA$ and $\mB$ are expressed as
\begin{equation}\label{eq:grad_ours}
    \nabla_\mA \gL = s\mQ_\mB^\top\mG,\:\nabla_\mB \gL = s\mG\mQ_\mA \:\:\:\:\:\:\text{(in \methodname)}.
\end{equation}
Equation~\ref{eq:grad_ours} shows that \textbf{the full gradient $\mG$ is backpropagated to the LoRA matrices only through the partial isometries $\mQ_\mA$ and $\mQ_\mB$, eliminating the negative effects of anisotropic scaling via singular values.}
This can be achieved without compromising the role of singular values in the forward pass, allowing $\Delta \mW_{\text{\methodname}}$ to represent any rank-$r$ matrix.

As shown in Figure~\ref{fig:motivation_effrank}, \methodname~overcomes the adverse effects of anisotropic gradient scaling, resulting in a higher effective rank in the gradients of $\mA$ and $\mB$ compared to LoRA.
Moreover, \methodname~narrows the gap between the full gradient $\mG$ and its low-rank approximation $\tilde{\mG}$.
To examine this, we derive the change in the effective weight under the \methodname~formulation, $\mW_\text{eff} = \mW_0 + \Delta \mW_\text{\methodname}$.
Although $\mQ_\mA$ and $\mQ_\mB$ are treated as constants during the backward pass, we periodically update them to reflect updates in the LoRA matrices, as described in the Training Procedure paragraph below.
Taking into account the variations in $\mQ_\mA$ and $\mQ_\mB$, the resulting change in $\mW_\text{eff}$ is given by $\Delta \mW_\text{eff} \approx s(\mQ_\mB \Delta \mA + \Delta \mB \mQ_\mA^\top + \Delta \mQ_\mB \mA + \mB \Delta \mQ_\mA^\top) = -\eta s^2(\mG \mQ_\mA \mQ_\mA^\top + \mQ_\mB \mQ_\mB^\top \mG) + s(\Delta \mQ_\mB \mA + \mB \Delta \mQ_\mA^\top)$.
Consequently, the effective gradient with respect to $\mW_\text{eff}$ is expressed as
\begin{equation}\label{eq:eff_grad_ours}
    \Tilde{\mG} = s^2(\mG \mQ_\mA \mQ_\mA^\top + \mQ_\mB \mQ_\mB^\top \mG) + \mE \:\:\:\:\:\text{(in SDS-LoRA)},
\end{equation}
where $\mE = -\frac{s}{\eta}(\Delta \mQ_\mB \mA + \mB \Delta \mQ_\mA^\top)$ collects the terms arising from the change in $\mQ_\mA$ and $\mQ_\mB$.
Because $\mQ_\mA\mQ_\mA^\top = \mP_{\gR(\mA)}$ and $\mQ_\mB\mQ_\mB^\top = \mP_{\gC(\mB)}$, our method attains the worst-case-optimal form of the gradient approximation characterized in Theorem~\ref{theorem:approx}, up to the error described by $\mE$.
In Section~\ref{sec:subspace_change}, we show that the contribution of $\mE$ is small: the ratio $||\mE||_F / ||\Tilde{\mG}||_F$ averages 0.09, 0.05, and 0.03 in the early, middle, and late phases of training, respectively (Table~\ref{tab:residual_e}).
Consistently, $\mQ_\mA$ and $\mQ_\mB$ change slowly, as the cosine similarity between $\mQ_\mA$ and $\mQ_\mA + \Delta \mQ_\mA$ (or $\mQ_\mB$ and $\mQ_\mB + \Delta \mQ_\mB$) remains close to 1 except for the first few iterations.
Thus, our method closely approaches the optimal $\Tilde{\mG}$, substantially reducing the gap to full fine-tuning.

\paragraph{Training Procedure.}

Following the standard initialization of $\mA$ and $\mB$, we sample $\mA$ from a uniform distribution and set $\mB = 0$.
Consequently, $\mQ_\mB$ cannot be defined from $\mB$ at the beginning of training.
To address this, we perform a warm-up training of $T$ iterations using the standard LoRA formulation.
After the warm-up, we compute the rank-$r$ truncated SVD of the weight updates, $\Delta \mW_\text{LoRA} = \mU_r\boldsymbol{\Sigma}_r\mV_r^\top$, and reparameterize it using the \methodname~formulation by setting the LoRA matrices and orthonormal bases as follows:
\begin{equation}\label{eq:initialization}
    \mA = \frac{1}{2 s} \boldsymbol{\Sigma}_r \mV_r^\top, \; \mB = \frac{1}{2 s} \mU_r \boldsymbol{\Sigma}_r,\; \mQ_\mA = \mV_r,\; \mQ_\mB = \mU_r.
\end{equation}
The coefficients in Equation~\ref{eq:initialization} are chosen so that the reparameterization exactly reproduces the warm-up result: substituting them into Equation~\ref{eq:method} gives $s(\mQ_\mB\mA + \mB\mQ_\mA^\top) = s(\frac{1}{2s}\mU_r\boldsymbol{\Sigma}_r\mV_r^\top + \frac{1}{2s}\mU_r\boldsymbol{\Sigma}_r\mV_r^\top) = \mU_r\boldsymbol{\Sigma}_r\mV_r^\top = \Delta \mW_\text{LoRA}$.
Using $\boldsymbol{\Sigma}_r^{1/2}$ instead would not reproduce this warm-up result, since Equation~\ref{eq:method} sums $\mQ_\mB\mA$ and $\mB\mQ_\mA^\top$ rather than multiplying $\mA$ and $\mB$ directly.
We reset all optimizer states after this process to ensure that the warm-up stage does not affect the subsequent main training stage.
During the main training stage, we recompute $\mQ_\mA$ and $\mQ_\mB$ as $\mA$ and $\mB$ evolve.
Although QR decomposition introduces minor computational overhead, we mitigate this cost by updating $\mQ_\mA$ and $\mQ_\mB$ intermittently.
To account for the larger learning rates early in training, we update them more frequently in the initial phase.
Specifically, at the $t$-th iteration, we update them if $t \bmod \left\lceil \frac{kt}{T_\text{total}} \right\rceil = 0$ where $k$ is a hyperparameter and $T_\text{total}$ indicates the total number of iterations.
This means that training is divided into $k$ phases, with the update interval in the $i$-th phase set to $i$.
See Algorithm~\ref{alg:training} for the overall training procedure.

{
\begin{algorithm}[t!]
\caption{Training Procedure of \methodname}
\label{alg:training}
\KwIn{Pretrained weight $\mW_0$, trainset $\mathcal{D}=\{\mathcal{B}_1, \dots \mathcal{B}_{T_\text{total}}\}$, $\mA$ and $\mB$ s.t. $\mB\mA = \boldsymbol{0}$, and warm-up iterations $T$.}
\tcp{\text{Warm-Up Stage}}
\For{$t=1$ \KwTo $T$}{
    $\Delta \mW_\text{LoRA} \leftarrow s\mB\mA$
    
    Forward pass on $\mathcal{B}_t$ using $\mW_\text{eff} = \mW_0+ \Delta \mW_\text{LoRA}$
    
    Backward pass and update $\mA$ and $\mB$
    
    \If{$t=T$}{
    Compute the rank-$r$ truncated SVD of $\Delta \mW_\text{LoRA}$: $\Delta \mW_\text{LoRA} = \mU_r \boldsymbol{\Sigma}_r \mV_r^\top$\\
    $\mA \leftarrow \frac{1}{2s} \boldsymbol{\Sigma}_r \mV_r^\top,\:\mB \leftarrow \frac{1}{2s} \mU_r \boldsymbol{\Sigma}_r,\:\mQ_\mA \leftarrow \mV_r,\:\mQ_\mB \leftarrow \mU_r\:$\tcp{\text{Equation~\ref{eq:initialization}}}
    
    Clear all optimizer states
    }    
}
\tcp{\text{Main Training Stage}}
\For{$t=1$ \KwTo $\normalfont T_\text{total}$}{
    $\Delta \mW_\text{\methodname} \leftarrow s(\mQ_\mB \mA + \mB \mQ_\mA^\top)\:$\tcp{\text{Equation~\ref{eq:method}}}
    
    Forward pass on $\mathcal{B}_t$ using $\mW_\text{eff} = \mW_0+ \Delta \mW_\text{\methodname}$
    
    Backward pass and update $\mA$ and $\mB$
    
    \If{$t \bmod \left\lceil \frac{kt}{T_\text{total}} \right\rceil = 0$}{
        Compute the QR decomposition of $\mA^\top$ and $\mB$: $\mA^\top = \mQ'_\mA \mR'_\mA$ and $\mB = \mQ'_\mB \mR'_\mB$
        $\mQ_\mA \leftarrow \mQ'_\mA,\:\mQ_\mB \leftarrow \mQ'_\mB$
    }
}
\Return $\mA, \mB$
\end{algorithm}
}

\begin{table*}[t!]
\caption{\textbf{Results on commonsense reasoning tasks.}
}
\vspace{-0.7em}
\label{tab:commonsense}
\centering
\resizebox{0.95\linewidth}{!}{
\begin{tabular}{p{1.8cm}l>{\centering\arraybackslash}p{0.5cm}ccccccccc}
    \toprule
     Model & Method & Rank & BoolQ & PIQA & SIQA & HellaSwag & WinoGrande & ARC-c & ARC-e & OBQA & Avg.\\
     \toprule
     \toprule
     \multirow{13}{*}{Gemma-2B} & Full FT. & - & 66.96{\scriptsize$\pm$0.21}&80.31{\scriptsize$\pm$0.12}&75.85{\scriptsize$\pm$0.47}&85.48{\scriptsize$\pm$0.43}&75.61{\scriptsize$\pm$0.67}&65.16{\scriptsize$\pm$0.70}&80.32{\scriptsize$\pm$0.38}&76.27{\scriptsize$\pm$0.34}&75.75\\
     \cmidrule{2-12}
     &LoRA&\multirow{5}{*}{8}&63.94{\scriptsize$\pm$0.04}&76.46{\scriptsize$\pm$0.28}&70.62{\scriptsize$\pm$0.36}&44.50{\scriptsize$\pm$0.48}&65.48{\scriptsize$\pm$0.74}&57.20{\scriptsize$\pm$0.28}&75.06{\scriptsize$\pm$0.56}&65.60{\scriptsize$\pm$1.41}&64.86\\
     &rsLoRA&&64.58{\scriptsize$\pm$0.07}&77.68{\scriptsize$\pm$0.56}&72.65{\scriptsize$\pm$0.60}&55.14{\scriptsize$\pm$1.73}&69.48{\scriptsize$\pm$0.15}&58.70{\scriptsize$\pm$0.00}&76.19{\scriptsize$\pm$0.38}&69.40{\scriptsize$\pm$0.57}&67.98\\
     &LoRA+&&63.35{\scriptsize$\pm$0.00}&76.82{\scriptsize$\pm$0.03}&72.00{\scriptsize$\pm$0.36}&77.18{\scriptsize$\pm$1.01}&70.90{\scriptsize$\pm$1.49}&58.47{\scriptsize$\pm$0.44}&76.06{\scriptsize$\pm$0.02}&70.20{\scriptsize$\pm$0.28}&70.62\\
     &PiSSA&&64.94{\scriptsize$\pm$0.20}&77.49{\scriptsize$\pm$0.13}&72.79{\scriptsize$\pm$0.31}&59.89{\scriptsize$\pm$3.06}&69.27{\scriptsize$\pm$0.04}&58.22{\scriptsize$\pm$0.52}&76.18{\scriptsize$\pm$0.12}&71.33{\scriptsize$\pm$0.19}&68.77\\
     &DoRA&&64.72{\scriptsize$\pm$0.16}&77.57{\scriptsize$\pm$0.41}&72.77{\scriptsize$\pm$0.36}&56.56{\scriptsize$\pm$0.61}&69.82{\scriptsize$\pm$0.41}&59.02{\scriptsize$\pm$0.56}&75.94{\scriptsize$\pm$0.08}&69.73{\scriptsize$\pm$0.47}&68.27\\
     \cmidrule{2-12}
     &\cellcolor{orange!25}\textbf{SDS-LoRA}&\cellcolor{orange!25}&\cellcolor{orange!25}\textbf{65.86}{\scriptsize $\pm$0.44}&\cellcolor{orange!25}\textbf{79.10}{\scriptsize $\pm$0.66}&\cellcolor{orange!25}\textbf{74.48}{\scriptsize $\pm$0.63}&\cellcolor{orange!25}\textbf{80.31}{\scriptsize $\pm$0.45}&\cellcolor{orange!25}\textbf{73.08}{\scriptsize $\pm$0.55}&\cellcolor{orange!25}\textbf{62.08}{\scriptsize $\pm$0.78}&\cellcolor{orange!25}\textbf{78.52}{\scriptsize $\pm$0.36}&\cellcolor{orange!25}\textbf{73.19}{\scriptsize $\pm$0.49}&\cellcolor{orange!25}\textbf{73.33}\\
     \cmidrule{2-12}
     &LoRA&\multirow{5}{*}{32}&64.24{\scriptsize$\pm$0.10}&76.22{\scriptsize$\pm$0.15}&70.71{\scriptsize$\pm$0.05}&45.05{\scriptsize$\pm$1.40}&65.98{\scriptsize$\pm$0.11}&57.51{\scriptsize$\pm$0.24}&74.96{\scriptsize$\pm$0.18}&65.13{\scriptsize$\pm$0.47}&64.98\\
     &rsLoRA&&65.63{\scriptsize$\pm$0.04}&78.89{\scriptsize$\pm$0.62}&73.97{\scriptsize$\pm$0.17}&77.69{\scriptsize$\pm$0.00}&71.03{\scriptsize$\pm$0.45}&61.26{\scriptsize$\pm$0.72}&78.55{\scriptsize$\pm$0.08}&72.80{\scriptsize$\pm$1.41}&72.48\\
     &LoRA+&&64.62{\scriptsize$\pm$0.14}&78.32{\scriptsize$\pm$0.10}&73.18{\scriptsize$\pm$0.65}&81.73{\scriptsize$\pm$0.10}&71.48{\scriptsize$\pm$0.45}&61.59{\scriptsize$\pm$0.20}&76.91{\scriptsize$\pm$0.30}&73.20{\scriptsize$\pm$0.28}&72.63\\
     &PiSSA&&65.74{\scriptsize$\pm$0.32}&77.98{\scriptsize$\pm$0.05}&74.14{\scriptsize$\pm$0.34}&72.90{\scriptsize$\pm$1.05}&70.96{\scriptsize$\pm$0.33}&59.47{\scriptsize$\pm$0.97}&76.91{\scriptsize$\pm$0.28}&74.40{\scriptsize$\pm$0.28}&71.56\\
     &DoRA&&65.30{\scriptsize$\pm$0.07}&78.89{\scriptsize$\pm$0.77}&74.19{\scriptsize$\pm$0.10}&74.17{\scriptsize$\pm$1.16}&71.56{\scriptsize$\pm$0.19}&61.06{\scriptsize$\pm$0.68}&78.56{\scriptsize$\pm$0.20}&72.60{\scriptsize$\pm$1.41}&72.04\\
     \cmidrule{2-12}
     &\cellcolor{orange!25}\textbf{SDS-LoRA}&\cellcolor{orange!25}&\cellcolor{orange!25}\textbf{66.34}{\scriptsize $\pm$0.35}&\cellcolor{orange!25}\textbf{80.78}{\scriptsize $\pm$0.53}&\cellcolor{orange!25}\textbf{74.86}{\scriptsize $\pm$0.58}&\cellcolor{orange!25}\textbf{84.93}{\scriptsize $\pm$0.52}&\cellcolor{orange!25}\textbf{74.01}{\scriptsize $\pm$0.64}&\cellcolor{orange!25}\textbf{62.84}{\scriptsize $\pm$0.85}&\cellcolor{orange!25}\textbf{79.34}{\scriptsize $\pm$0.44}&\cellcolor{orange!25}\textbf{75.29}{\scriptsize $\pm$0.82}&\cellcolor{orange!25}\textbf{74.80}\\
     \midrule
     \multirow{13}{*}{LLaMA3-8B} & Full FT. & - & 75.24{\scriptsize$\pm$0.21}& 89.70{\scriptsize$\pm$0.22} & 81.51{\scriptsize$\pm$0.24} & 96.11{\scriptsize$\pm$0.15} & 87.66{\scriptsize$\pm$0.17} & 81.23{\scriptsize$\pm$0.54} & 92.07{\scriptsize$\pm$0.19} & 88.02{\scriptsize$\pm$0.31} & 86.44\\
     \cmidrule{2-12}
     &LoRA&\multirow{5}{*}{8}&72.98{\scriptsize$\pm$0.06}&87.60{\scriptsize$\pm$0.31}&79.67{\scriptsize$\pm$0.39}&94.35{\scriptsize$\pm$0.14}&83.61{\scriptsize$\pm$0.15}&78.95{\scriptsize$\pm$0.44}&90.14{\scriptsize$\pm$0.02}&83.87{\scriptsize$\pm$1.04}&83.89\\
     &rsLoRA&&73.40{\scriptsize$\pm$0.19}&88.23{\scriptsize$\pm$0.26}&80.26{\scriptsize$\pm$0.41}&95.19{\scriptsize$\pm$0.03}&84.98{\scriptsize$\pm$0.19}&79.21{\scriptsize$\pm$0.32}&90.70{\scriptsize$\pm$0.12}&84.80{\scriptsize$\pm$0.28}&84.60\\
     &LoRA+&&73.07{\scriptsize$\pm$0.61}&88.03{\scriptsize$\pm$0.44}&80.22{\scriptsize$\pm$0.29}&94.75{\scriptsize$\pm$0.01}&85.00{\scriptsize$\pm$0.07}&78.95{\scriptsize$\pm$0.36}&90.16{\scriptsize$\pm$0.06}&85.20{\scriptsize$\pm$0.00}&84.42\\
     &PiSSA&&73.68{\scriptsize$\pm$0.01}&88.16{\scriptsize$\pm$0.10}&80.37{\scriptsize$\pm$0.19}&95.20{\scriptsize$\pm$0.05}&85.82{\scriptsize$\pm$0.19}&\textbf{80.12}{\scriptsize$\pm$0.60}&90.36{\scriptsize$\pm$0.06}& 85.67{\scriptsize$\pm$0.66} &84.92\\
     &DoRA&&73.20{\scriptsize$\pm$0.01}&87.85{\scriptsize$\pm$0.10}&80.21{\scriptsize$\pm$0.27}&95.22{\scriptsize$\pm$0.00}&84.37{\scriptsize$\pm$0.45}&79.66{\scriptsize$\pm$0.16}&90.53{\scriptsize$\pm$0.12}&85.53{\scriptsize$\pm$0.19}&84.57\\
     \cmidrule{2-12}
     &\cellcolor{orange!25}\textbf{SDS-LoRA}&\cellcolor{orange!25}&\cellcolor{orange!25}\textbf{73.91}{\scriptsize $\pm$0.41}&\cellcolor{orange!25}\textbf{88.77}{\scriptsize $\pm$0.80}&\cellcolor{orange!25}\textbf{80.91}{\scriptsize $\pm$0.40}&\cellcolor{orange!25}\textbf{96.00}{\scriptsize $\pm$0.57}&\cellcolor{orange!25}\textbf{87.08}{\scriptsize $\pm$0.52}&\cellcolor{orange!25}79.90{\scriptsize $\pm$0.45}&\cellcolor{orange!25}\textbf{90.80}{\scriptsize $\pm$0.74}&\cellcolor{orange!25}\textbf{86.02}{\scriptsize $\pm$0.48}&\cellcolor{orange!25}\textbf{85.42}\\
     \cmidrule{2-12}
     &LoRA&\multirow{5}{*}{32}&73.06{\scriptsize$\pm$0.13}&87.45{\scriptsize$\pm$0.05}&79.68{\scriptsize$\pm$0.00}&94.50{\scriptsize$\pm$0.08}&83.32{\scriptsize$\pm$0.15}&79.58{\scriptsize$\pm$0.52}&90.31{\scriptsize$\pm$0.02}&84.07{\scriptsize$\pm$0.75}&84.00\\
     &rsLoRA&&72.56{\scriptsize$\pm$0.89}&88.47{\scriptsize$\pm$0.08}&80.74{\scriptsize$\pm$0.39}&91.80{\scriptsize$\pm$5.36}&85.61{\scriptsize$\pm$0.52}&79.69{\scriptsize$\pm$0.60}&90.90{\scriptsize$\pm$0.08}&85.53{\scriptsize$\pm$1.04}&84.41\\
     &LoRA+&&73.43{\scriptsize$\pm$0.30}&88.25{\scriptsize$\pm$0.10}&80.03{\scriptsize$\pm$0.05}&94.99{\scriptsize$\pm$0.01}&85.87{\scriptsize$\pm$0.41}&79.92{\scriptsize$\pm$0.20}&90.05{\scriptsize$\pm$0.32}&85.73{\scriptsize$\pm$1.04}&84.78\\
     &PiSSA&&74.46{\scriptsize$\pm$0.52}&89.03{\scriptsize$\pm$0.44}&80.79{\scriptsize$\pm$0.41}&95.48{\scriptsize$\pm$0.02}&86.98{\scriptsize$\pm$0.11}&80.12{\scriptsize$\pm$0.36}&90.75{\scriptsize$\pm$0.14}&\textbf{86.33}{\scriptsize$\pm$0.66}&85.49\\
     &DoRA&&73.86{\scriptsize$\pm$0.23}&88.72{\scriptsize$\pm$0.33}&80.98{\scriptsize$\pm$0.27}&95.67{\scriptsize$\pm$0.05}&85.85{\scriptsize$\pm$0.41}&79.78{\scriptsize$\pm$0.48}&91.08{\scriptsize$\pm$0.00}&85.87{\scriptsize$\pm$0.09}&85.23\\
     \cmidrule{2-12}
     &\cellcolor{orange!25}\textbf{SDS-LoRA}&\cellcolor{orange!25}&\cellcolor{orange!25}\textbf{74.69}{\scriptsize $\pm$0.49}&\cellcolor{orange!25}\textbf{89.43}{\scriptsize $\pm$0.66}&\cellcolor{orange!25}\textbf{81.05}{\scriptsize $\pm$0.58}&\cellcolor{orange!25}\textbf{96.21}{\scriptsize $\pm$0.75}&\cellcolor{orange!25}\textbf{88.08}{\scriptsize $\pm$0.55}&\cellcolor{orange!25}\textbf{81.03}{\scriptsize $\pm$0.77}&\cellcolor{orange!25}\textbf{92.01}{\scriptsize $\pm$0.34}&\cellcolor{orange!25}86.18{\scriptsize $\pm$0.78}&\cellcolor{orange!25}\textbf{86.08}\\
     \bottomrule
\end{tabular}}
\vspace{-1em}
\end{table*}

\begin{table*}[t!]
\centering
\begin{minipage}{0.6\textwidth}
\caption{\textbf{Results on natural language generation tasks.}}
\vspace{-0.5em}
\label{tab:nlg}
\centering
\resizebox{1\linewidth}{!}{
\begin{tabular}{p{1.8cm}>{\centering\arraybackslash}p{0.5cm}cccccc}
    \toprule
    \multirow{2}{*}{Method} & \multirow{2}{*}{Rank} & \multicolumn{3}{c}{Gemma-2B} & \multicolumn{3}{c}{LLaMA3-8B} \\
    \cmidrule(lr){3-5} \cmidrule(lr){6-8}
    & & MATH & GSM8K & HumanEval & MATH & GSM8K & HumanEval \\
    \toprule
    \toprule
    Full FT. & - & 19.17{\scriptsize$\pm$0.22} & 56.23{\scriptsize$\pm$0.23} & 33.35{\scriptsize$\pm$0.65} & 26.47{\scriptsize$\pm$0.41} & 77.04{\scriptsize$\pm$0.18} & 49.11{\scriptsize$\pm$0.44} \\
    \midrule
    LoRA   & \multirow{5}{*}{8}  & 16.12{\scriptsize$\pm$0.28} & 45.59{\scriptsize$\pm$0.95} & 27.24{\scriptsize$\pm$0.57} & 23.68{\scriptsize$\pm$0.28} & 73.44{\scriptsize$\pm$0.22} & 42.53{\scriptsize$\pm$1.80} \\
    rsLoRA && 17.10{\scriptsize$\pm$0.24} & 49.10{\scriptsize$\pm$1.44} & 29.27{\scriptsize$\pm$1.79} & 24.52{\scriptsize$\pm$0.09} & 75.71{\scriptsize$\pm$0.23} & 42.23{\scriptsize$\pm$1.63} \\
    LoRA+  && 17.09{\scriptsize$\pm$0.46} & 50.30{\scriptsize$\pm$0.81} & 27.64{\scriptsize$\pm$0.76} & 24.65{\scriptsize$\pm$0.63} & 75.27{\scriptsize$\pm$0.59} & 45.27{\scriptsize$\pm$2.84} \\
    PiSSA  && 16.44{\scriptsize$\pm$0.38} & 50.97{\scriptsize$\pm$0.74} & 28.86{\scriptsize$\pm$0.58} & 24.43{\scriptsize$\pm$0.23} & 74.96{\scriptsize$\pm$1.01} & 46.19{\scriptsize$\pm$2.64} \\
    DoRA   && 17.19{\scriptsize$\pm$0.28} & 50.61{\scriptsize$\pm$0.85} & 28.35{\scriptsize$\pm$1.52} & 24.63{\scriptsize$\pm$0.21} & 75.41{\scriptsize$\pm$1.32} & 43.29{\scriptsize$\pm$1.14} \\
    \midrule
    \cellcolor{orange!25}\textbf{SDS-LoRA} & \cellcolor{orange!25} & 
    \cellcolor{orange!25}\textbf{17.62}{\scriptsize $\pm$0.66}&\cellcolor{orange!25}\textbf{52.01}{\scriptsize $\pm$0.72}&\cellcolor{orange!25}\textbf{31.24}{\scriptsize $\pm$1.43}&
    \cellcolor{orange!25}\textbf{25.81}{\scriptsize $\pm$0.53}&\cellcolor{orange!25}\textbf{76.42}{\scriptsize $\pm$0.90}&\cellcolor{orange!25}\textbf{46.95}{\scriptsize $\pm$1.56} \\
    \midrule
    LoRA   & \multirow{5}{*}{32} & 16.31{\scriptsize$\pm$0.13} & 45.67{\scriptsize$\pm$0.91} & 28.66{\scriptsize$\pm$0.50} & 23.66{\scriptsize$\pm$0.25} & 73.49{\scriptsize$\pm$0.59} & 42.53{\scriptsize$\pm$1.09} \\
    rsLoRA && 18.09{\scriptsize$\pm$0.04} & 51.93{\scriptsize$\pm$0.53} & 31.10{\scriptsize$\pm$1.32} & 25.97{\scriptsize$\pm$0.17} & 76.28{\scriptsize$\pm$0.96} & 44.36{\scriptsize$\pm$1.39} \\
    LoRA+  && 17.45{\scriptsize$\pm$0.39} & 52.74{\scriptsize$\pm$0.75} & 31.30{\scriptsize$\pm$1.25} & 25.01{\scriptsize$\pm$0.11} & 76.11{\scriptsize$\pm$0.20} & 46.34{\scriptsize$\pm$1.29} \\
    PiSSA  && 17.23{\scriptsize$\pm$0.24} & 53.38{\scriptsize$\pm$0.77} & 32.11{\scriptsize$\pm$0.76} & 24.80{\scriptsize$\pm$0.41} & 76.39{\scriptsize$\pm$0.75} & 47.10{\scriptsize$\pm$2.38} \\
    DoRA   && 18.27{\scriptsize$\pm$0.14} & 52.41{\scriptsize$\pm$0.24} & 30.79{\scriptsize$\pm$1.52} & 25.87{\scriptsize$\pm$0.42} & 76.42{\scriptsize$\pm$0.17} & 44.97{\scriptsize$\pm$1.00} \\
    \midrule
    \cellcolor{orange!25}\textbf{SDS-LoRA} & \cellcolor{orange!25} & 
    \cellcolor{orange!25}\textbf{18.31}{\scriptsize $\pm$0.54}&\cellcolor{orange!25}\textbf{55.21}{\scriptsize $\pm$0.81}&\cellcolor{orange!25}\textbf{34.01}{\scriptsize $\pm$1.50}& 
    \cellcolor{orange!25}\textbf{26.36}{\scriptsize $\pm$0.73}&\cellcolor{orange!25}\textbf{77.13}{\scriptsize $\pm$0.84}&\cellcolor{orange!25}\textbf{47.61}{\scriptsize $\pm$1.67}\\
    \bottomrule
\end{tabular}}
\end{minipage}
\begin{minipage}{0.32\textwidth}
    \caption{\textbf{Results on image classification tasks with ViT-Base.}
    See Section~\ref{sec:vision} for details and more results.
    }
\vspace{-0.5em}
\label{tab:vision_main}
\centering
\resizebox{1\linewidth}{!}{
\begin{tabular}{l>{\centering\arraybackslash}p{0.5cm}ccc}
    \toprule
    Method & Rank & Cars & CUB200 & SUN397\\
    \toprule
    \toprule
    Full FT. & - & 84.26{\scriptsize$\pm$0.19} & 86.35{\scriptsize$\pm$0.18} & 74.76{\scriptsize$\pm$0.26} \\
    \cmidrule{1-5}
    LoRA & \multirow{5}{*}{8} & 78.66{\scriptsize$\pm$0.23} & 85.65{\scriptsize$\pm$0.05} & 74.35{\scriptsize$\pm$0.28}\\
    rsLoRA & & 78.48{\scriptsize$\pm$0.07} & 85.67{\scriptsize$\pm$0.41} & 74.63{\scriptsize$\pm$0.13}\\
    LoRA+ & & 79.05{\scriptsize$\pm$0.25} & 85.28{\scriptsize$\pm$0.11} & 72.87{\scriptsize$\pm$0.28}\\
    PiSSA & & 78.03{\scriptsize$\pm$0.32} & 85.17{\scriptsize$\pm$0.24} & 72.82{\scriptsize$\pm$0.10}\\
    DoRA & & 78.93{\scriptsize$\pm$0.30} & 85.74{\scriptsize$\pm$0.35} & 74.46{\scriptsize$\pm$0.09}\\
    \cmidrule{1-5}
    \cellcolor{orange!25}\textbf{SDS-LoRA} & \cellcolor{orange!25} & \cellcolor{orange!25}\textbf{80.41}{\scriptsize $\pm$0.20}&\cellcolor{orange!25}\textbf{86.07}{\scriptsize $\pm$0.13}&\cellcolor{orange!25}\textbf{74.85}{\scriptsize $\pm$0.13}\\
    \cmidrule{1-5}
    LoRA & \multirow{5}{*}{32} & 81.48{\scriptsize$\pm$0.20} & 85.75{\scriptsize$\pm$0.28} & 70.82{\scriptsize$\pm$0.07}\\
    rsLoRA & & 80.57{\scriptsize$\pm$0.22} & 85.42{\scriptsize$\pm$0.41} & 73.92{\scriptsize$\pm$0.26}\\
    LoRA+ & & 79.53{\scriptsize$\pm$0.60} & 85.68{\scriptsize$\pm$0.31} & 67.66{\scriptsize$\pm$0.06}\\
    PiSSA & & 80.43{\scriptsize$\pm$0.17} & 84.90{\scriptsize$\pm$0.28} & 69.73{\scriptsize$\pm$0.12}\\
    DoRA & & 81.07{\scriptsize$\pm$0.03} & 85.67{\scriptsize$\pm$0.22} & 74.05{\scriptsize$\pm$0.21}\\
    \cmidrule{1-5}
    \cellcolor{orange!25}\textbf{SDS-LoRA} & \cellcolor{orange!25} & \cellcolor{orange!25}\textbf{82.95}{\scriptsize$\pm$0.25} & \cellcolor{orange!25}\textbf{85.99}{\scriptsize$\pm$0.22} & \cellcolor{orange!25}\textbf{74.89}{\scriptsize$\pm$0.14}\\
    \bottomrule
\end{tabular}

}
\end{minipage}
\vspace{-1em}
\end{table*}

\subsection{Convergence Analysis}\label{sec:convergence}
In Section~\ref{sec:lora_limitation}, we utilize the descent lemma to formally characterize the adverse effects of anisotropic gradient scaling.
To demonstrate how overcoming this issue improves loss convergence, we extend this analysis to compare the convergence rates of LoRA and \methodname.

We first assume that the loss function with respect to $\mW_\text{eff}$ satisfies the Polyak–Łojasiewicz (PL) condition, a standard assumption in non-convex optimization used to establish linear convergence rates: $\|\nabla_{\mW_\text{eff}} \gL\|_F^2 \geq 2\mu(\gL(\mW_\text{eff}) - \gL^*)$ where $\mu > 0$ and $\gL^*$ denotes the minimum loss.
While this condition can be restrictive in general deep learning contexts, it is plausible when fine-tuning from a pretrained model, where the loss landscape is typically better conditioned~\citep{xu2025pl}.
Accordingly, we assume that the PL condition holds in a neighborhood of the pretrained weight $\mW_0$.
We also assume that the change in $\mQ_\mA$ and $\mQ_\mB$ is negligible, resulting in $\mE \approx 0$ in Equation~\ref{eq:eff_grad_ours}.
See Section~\ref{sec:subspace_change} for empirical validation of this assumption.
Additionally, let $\alpha \in [0,2]$ represent the alignment between $\mG$ and the low-rank subspaces $\gR(\mA)$ and $\gC(\mB)$ such that: $\|\mG \mP_{\gR(\mA)}\|_F^2 + \|\mP_{\gC(\mB)}\mG\|_F^2=\alpha \|\mG\|_F^2$.
Under these assumptions, we derive the following theorem regarding the linear convergence rates of LoRA and \methodname:
\begin{theorem}\label{theorem:convergence}
    Assume that the $\mu$-PL condition holds near the pretrained weight $\mW_0$, that the loss $\gL$ is $\beta$-smooth, and that the subspaces of $\mA$ and $\mB$ capture a fraction $\alpha$ of the energy of $\mG$.
    Then, the linear convergence rates of standard LoRA and \methodname~are given by:
    \begin{equation}
        \gL_{t+1} - \gL^* \leq
        \begin{cases}
            (1-\frac{\mu\alpha}{2\beta\kappa^4})(\gL_t- \gL^*) & \mbox {\normalfont(LoRA)} \\
            (1-\frac{\mu\alpha}{2\beta})(\gL_t- \gL^*) & \mbox {\normalfont(\methodname)}
        \end{cases}
    \end{equation}
    where $\gL^*$ denotes the minimal loss and $\kappa = \frac{\max(\sigma_{\normalfont\text{max}}(\mA),\sigma_{\normalfont\text{max}}(\mB))}{\min(\sigma_{\normalfont\text{min}}(\mA),\sigma_{\normalfont\text{min}}(\mB))}$.
\end{theorem}
See Section~\ref{sec:convergence_proof} for the proof and details.
Theorem~\ref{theorem:convergence} suggests that as the singular values of $\mA$ and $\mB$ become more skewed, the condition number $\kappa$ increases, thereby degrading the convergence rate of LoRA.
In contrast, since singular values do not influence the backward pass in \methodname, its convergence rate is independent of $\kappa$.
Our theoretical analysis further highlights the significance of addressing anisotropic gradient scaling in LoRA.

\section{Experiments}

\begin{table*}[t!]
\caption{\textbf{Comparison with related methods.}
We reproduce the results of existing methods.
We set the rank to 8.
}
\vspace{-0.5em}
\label{tab:gradientbased_r8}
\centering
\resizebox{0.8\linewidth}{!}{
\begin{tabular}{lcccccc}
     \toprule
     \multirow{2}{*}{Method} & \multicolumn{3}{c}{Gemma-2B} & \multicolumn{3}{c}{LLaMA3-8B}\\
     \cmidrule(lr){2-4} \cmidrule(lr){5-7}
     & MATH & GSM8K & HumanEval & MATH & GSM8K & HumanEval\\
     \toprule
     \toprule
     LoRA-GA & 17.02{\scriptsize$\pm$0.28} & 48.60{\scriptsize$\pm$0.74} & 30.69{\scriptsize$\pm$0.76} & 24.74{\scriptsize$\pm$0.23} & 75.21{\scriptsize$\pm$1.01} & 45.12{\scriptsize$\pm$1.22}\\
     LoRA-Pro & 16.52{\scriptsize$\pm$0.14} & 44.58{\scriptsize$\pm$0.55} & 28.66{\scriptsize$\pm$0.99} & 23.90{\scriptsize$\pm$0.44} & 72.93{\scriptsize$\pm$0.17} & 40.24{\scriptsize$\pm$0.99} \\
     rsLoRA + ScaledAdamW & 16.64{\scriptsize$\pm$0.31} & 45.56{\scriptsize$\pm$0.53} & 26.83{\scriptsize$\pm$0.99} & 24.00{\scriptsize$\pm$0.17} & 73.39{\scriptsize$\pm$0.20} & 43.29{\scriptsize$\pm$1.15} \\
     AltLoRA & 16.62{\scriptsize$\pm$0.28} & 45.34{\scriptsize$\pm$0.81} & 28.66{\scriptsize$\pm$0.91} & 23.66{\scriptsize$\pm$0.68} & 73.69{\scriptsize$\pm$0.23} & 42.07{\scriptsize$\pm$0.82} \\
     \midrule
     \cellcolor{orange!25}\textbf{\methodname} & 
    \cellcolor{orange!25}\textbf{17.62}{\scriptsize $\pm$0.66}&\cellcolor{orange!25}\textbf{52.01}{\scriptsize $\pm$0.72}&\cellcolor{orange!25}\textbf{31.24}{\scriptsize $\pm$1.43}&
    \cellcolor{orange!25}\textbf{25.81}{\scriptsize $\pm$0.53}&\cellcolor{orange!25}\textbf{76.42}{\scriptsize $\pm$0.90}&\cellcolor{orange!25}\textbf{46.95}{\scriptsize $\pm$1.56}\\
     \bottomrule
\end{tabular}}
\vspace{-1em}
\end{table*}

\begin{figure}[t!]
    \centering
    \begin{subfigure}[b]{0.48\textwidth}
        \centering
        \includegraphics[width=0.49\linewidth]{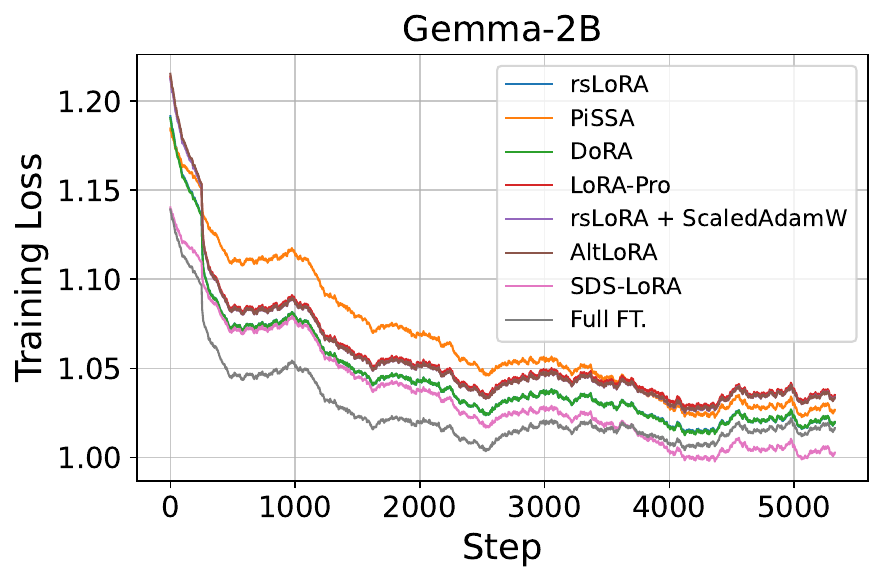}
        \hfill
        \includegraphics[width=0.49\linewidth]{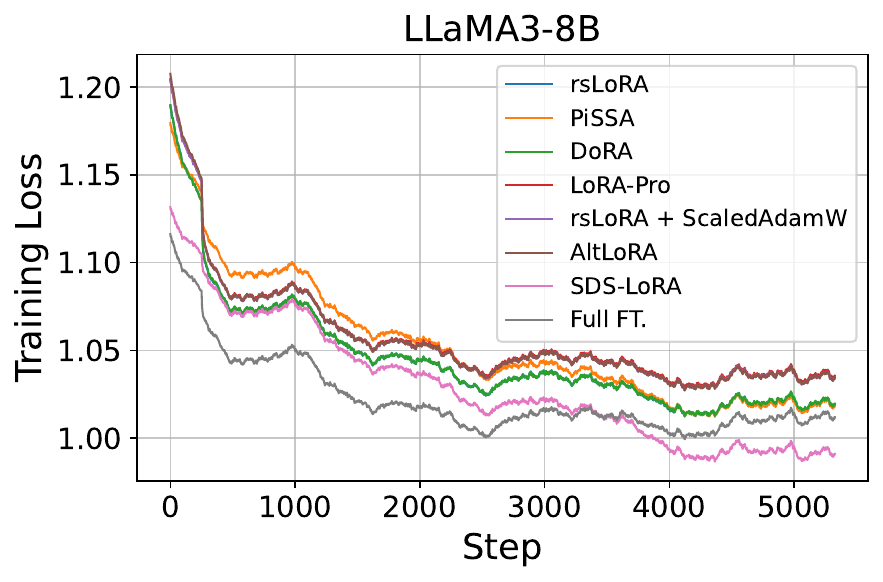}
        \caption{Training loss curve}
        \label{fig:loss_curve}
    \end{subfigure}
    \hspace{0.5em}
    \begin{subfigure}[b]{0.48\textwidth}
        \centering
        \includegraphics[width=0.49\linewidth]{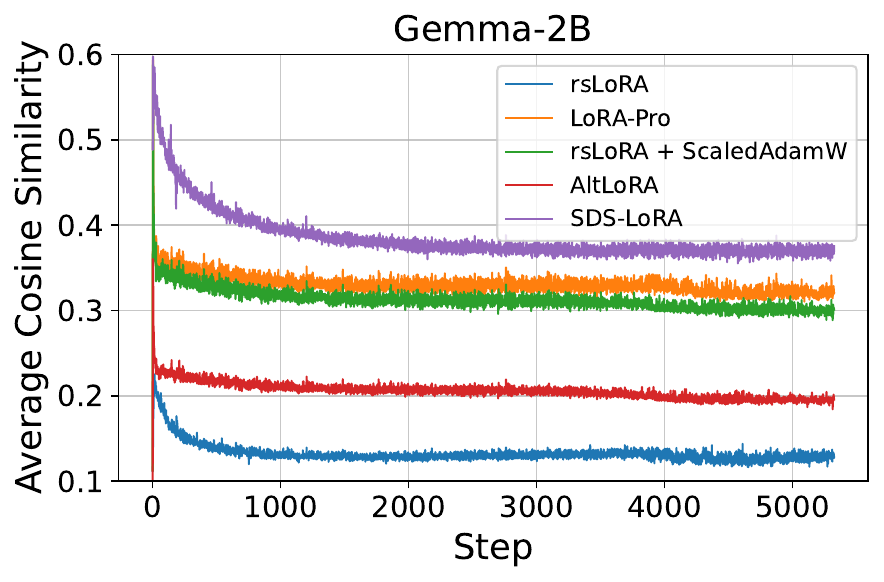}
        \hfill
        \includegraphics[width=0.49\linewidth]{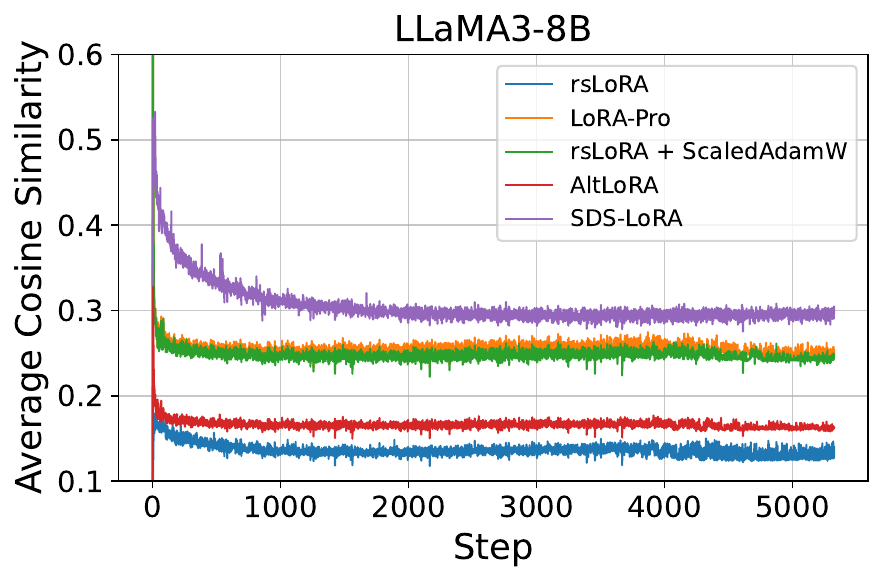}
        \caption{Alignment between $\mG$ and $\Tilde{\mG}$}
        \label{fig:cossim}
    \end{subfigure}
    \caption{\textbf{Comparison of loss convergence and gradient approximation quality.}
    (a) We smooth the loss curve to enhance visual clarity.
    (b) We report the average cosine similarity between $\mG$ and $\Tilde{\mG}$ across all layers.
    Models are fine-tuned using the Commonsense-170K dataset with rank 32.
    }
    \label{fig:convergence_alignment}
\end{figure}

\subsection{Experimental Setup}\label{sec:experimental_setup}
\paragraph{Models and Datasets.}
For commonsense reasoning tasks, we fine-tune on the Commonsense-170K dataset~\citep{hu2023llm} and evaluate on eight standard benchmarks, namely BoolQ~\citep{clark2019boolq}, PIQA~\citep{Bisk2020piqa}, SIQA~\citep{sap2019socialqa}, HellaSwag~\citep{zellers-etal-2019-hellaswag}, WinoGrande~\citep{Sakaguchi2020winogrande}, ARC-c/e~\citep{clark2018arc}, and OBQA~\citep{mihaylov2018obqa}.
For natural language generation tasks, we select 100K examples from the MetaMathQA~\citep{Yu2024metamath} and Code-Feedback~\citep{zheng2024codefeedback} datasets for fine-tuning.
To assess performance, we use MATH~\citep{Hendrycks2021math} and GSM8K~\citep{cobbe2021gsm8k} benchmarks for the models fine-tuned on MetaMathQA, and HumanEval~\citep{humaneval} benchmark for the models fine-tuned on Code-Feedback.
For the details on image classification tasks, see Section~\ref{sec:vision}.
\paragraph{Implementation Details.}
For optimization, we use the AdamW optimizer~\citep{Loshchilov2019adamw} with standard settings.
We set $\beta_1 = 0.9$, $\beta_2 = 0.999$, and apply zero weight decay to the optimizer.
For SDS-LoRA, we set the warm-up iteration $T$ to 10 and the update scheduling hyperparameter $k$ to 5.
For the scaling factor $s$, we fix it to 2 for LoRA and $\frac{\alpha}{\sqrt{r}}$ for the other methods, following rsLoRA~\citep{kalajdzievski2023rslora}, where $\alpha$ is set to 4.
We compare our method with the original LoRA and its recent variants, including rsLoRA~\citep{kalajdzievski2023rslora}, LoRA+~\citep{Hayou2024loraplus} with a learning rate ratio of 4, PiSSA~\citep{meng2024pissa}, and DoRA~\citep{liu2024dora}.
We reproduce the results of these methods under the same setting.
We report the mean and standard deviation over three trials.
All experiments are conducted on NVIDIA H200 GPUs.
More details are provided in Table~\ref{tab:additional_detail}.

\subsection{Commonsense Reasoning}
Table~\ref{tab:commonsense} shows the experimental results on commonsense reasoning tasks.
\methodname~outperforms LoRA and its variants in nearly all configurations, leading to the highest adaptation performance on average across all models and ranks.
Specifically, \methodname~yields improvements of approximately 9 and 2 percentage points over LoRA on Gemma-2B and LLaMA3-8B, respectively, and surpasses recent variants including PiSSA and DoRA by 3–5 and 0.5–1 percentage points on Gemma-2B and LLaMA3-8B, respectively.
In particular, Gemma-2B with \methodname~shows a substantial improvement over the others on the challenging HellaSwag benchmark, achieving about a 40 percentage point improvement over LoRA.
Compared to full fine-tuning, \methodname~with rank-32 achieves minimal performance gaps, demonstrating the effectiveness of \methodname.

\subsection{Natural Language Generation}
Table~\ref{tab:nlg} shows the experimental results on the natural language generation tasks.
\methodname~achieves the best adaptation performance on all benchmarks across all models and ranks.
Specifically, \methodname~achieves improvements of approximately 1–2, 1–10, and 2–6 percentage points over existing methods on MATH, GSM8K, and HumanEval, respectively.
Consistent with the commonsense reasoning results, \methodname~demonstrates substantial performance gains with Gemma-2B, especially in challenging adaptation contexts.
Furthermore, \methodname~with rank 32 achieves results comparable to or exceeding those of full fine-tuning.
These results demonstrate that \methodname~significantly improves LoRA's performance.

\subsection{Image Classification}
In addition to natural language processing tasks, we also provide results for image classification tasks.
The results in Table~\ref{tab:vision_main} demonstrate that our method improves the performance of LoRA substantially, surpassing that of the other methods by significant margins.
In particular, we observe that \methodname~achieves about 2-3 percentage points higher accuracy on the Cars dataset, resulting in the smallest performance gap compared to the full fine-tuning case.
These results demonstrate the effectiveness of \methodname~in the vision domain.
See Section~\ref{sec:vision} for experimental details and more results.

\subsection{Comparison with Related Methods}\label{sec:comparison_gradient}
We compare \methodname~with existing methods that view LoRA from the low-rank gradient approximation perspective, including LoRA-GA~\citep{wang2024loraga}, LoRA-Pro~\citep{wang2025lorapro}, ScaledAdamW~\citep{riemannian2024zhang}, and AltLoRA~\citep{altlora2025yu}.
We reproduce the results of these methods with the same experimental setup as detailed in Table~\ref{tab:additional_detail}.
For a fair comparison, we implement LoRA-Pro without tracking the momentum of full fine-tuning gradients.
Table~\ref{tab:gradientbased_r8} shows that \methodname~outperforms existing methods by a substantial margin.
We attribute these results to the fact that while both \methodname~and the previous methods optimize the update rule in the effective weight space, the previous ones fail to account for anisotropic scaling in the LoRA matrices' gradient, leaving them vulnerable to its adverse effects.
We further compare \methodname~with LoRA-RITE~\citep{yen2025lora} and RefLoRA~\citep{zhang2025reflora} in Section~\ref{sec:precond_discussion}, where we also provide further discussion.

\subsection{Loss Convergence and Gradient Approximation Quality}\label{sec:loss_approximation}
Figure~\ref{fig:loss_curve} shows the training loss curves for full fine-tuning, \methodname, and existing methods.
The results demonstrate that \methodname~achieves significantly better loss convergence than both standard LoRA (i.e., rsLoRA) and related methods.
Notably, our method achieves a lower final training loss than full fine-tuning.
Figure~\ref{fig:cossim} illustrates the average cosine similarity between the full gradient $\mG$ and its low-rank approximation $\Tilde{\mG}$ across all layers.
We observe that \methodname~exhibits higher alignment throughout training than the existing methods.
As discussed in Section~\ref{sec:precond_discussion}, we attribute the performance gap between SDS-LoRA and the compared methods (i.e., preconditioning-based methods) to the fact that our method overcomes anisotropic gradient scaling in the LoRA matrices' gradients, whereas the preconditioning-based methods do not.
This highlights the significance of addressing anisotropic gradient scaling.
Overall, these results support our theoretical analysis of both the quality of the gradient approximation (Theorem~\ref{theorem:approx}) and the convergence rate (Theorem~\ref{theorem:convergence}).

\subsection{Ablation Studies}\label{sec:ablation}
To justify the design choice of SDS-LoRA $\Delta \mW_\text{\methodname}=s(\mQ_\mB\mA + \mB \mQ_\mA^\top)$, we compare it against the LoRA formulation and a variant where weight updates are defined as $\Delta \mW = s\mQ_\mB\mQ_\mA^\top$, with both $\mQ_\mA$ and $\mQ_\mB$ being learnable.
Specifically, for this variant, we compute the QR decompositions of $\mQ_\mA$ and $\mQ_\mB$ after their optimization step and replace them with new ones to ensure that their columns remain orthonormal.
Although this approach prevents anisotropic gradient scaling, its representational capacity is severely limited because the non-zero singular values of the weight updates are forced to be uniform.
In contrast, \methodname~eliminates anisotropic gradient scaling without compromising the representational capacity of weight updates.
As shown in Table~\ref{tab:ablation}, \methodname~significantly outperforms both standard LoRA and the variant, each of which satisfies only one of the two desirable properties.
These results suggest that addressing anisotropic gradient scaling in LoRA is not straightforward due to the dual role of singular values in the forward and backward passes.
Additional ablation studies, on orthogonalization-based variants (Section~\ref{sec:orthogonal_lora}) and on the warm-up stage (Section~\ref{sec:warmup_ablation}), are provided in the appendix.

\subsection{Training Overhead}\label{sec:overhead}
We measure the training time and GPU memory of \methodname~relative to LoRA.
Table~\ref{tab:overhead} shows that \methodname~introduces approximately 5\% overhead in time and less than 1\% overhead in GPU-memory.
These results further highlight the efficiency of \methodname.

\begin{table*}[t!]
\caption{\textbf{Ablation studies.}
`Full RC of $\Delta \mW$' indicates the full representational capacity of $\Delta \mW$.
Models are fine-tuned with rank 32.
}
\vspace{-0.5em}
\centering
\resizebox{0.95\linewidth}{!}{
\begin{tabular}{lcccccccc}
     \toprule
     \multirow{2}{*}{Formulation of $\Delta \mW$} & \multirow{2}{*}{Anisotropic Gradient Scaling} & \multirow{2}{*}{Full RC of $\Delta \mW$} & \multicolumn{3}{c}{Gemma-2B} & \multicolumn{3}{c}{LLaMA3-8B}\\
     \cmidrule(lr){4-6} \cmidrule(lr){7-9}
     &&& MATH & GSM8K & HumanEval & MATH & GSM8K & HumanEval\\
     \toprule
     \toprule
     Full FT. & - & - & 19.17{\scriptsize$\pm$0.22} & 56.23{\scriptsize$\pm$0.23} & 33.35{\scriptsize$\pm$0.65} & 26.47{\scriptsize$\pm$0.41} & 77.04{\scriptsize$\pm$0.18} & 49.11{\scriptsize$\pm$0.44}\\
     \midrule
     $s\mB\mA$ &\textcolor{red}{\ding{52}}&\textcolor{green}{\ding{52}}& 18.09{\scriptsize$\pm$0.04} & 51.93{\scriptsize$\pm$0.53}& 31.10{\scriptsize$\pm$1.32} &25.97{\scriptsize$\pm$0.17} & 76.28{\scriptsize$\pm$0.96} & 44.36{\scriptsize$\pm$1.39}\\
     $s\mQ_\mB\mQ_\mA^\top$ (trainable $\mQ_\mA$ and $\mQ_\mB$)&\textcolor{green}{\ding{56}}& \textcolor{red}{\ding{56}}& 17.22{\scriptsize$\pm$0.24} & 51.18{\scriptsize$\pm$0.29}& 31.30{\scriptsize$\pm$0.99} &24.14{\scriptsize$\pm$0.27} & 75.68{\scriptsize$\pm$0.35} & 45.04{\scriptsize$\pm$1.02}\\
     \midrule
     \cellcolor{orange!25}$s(\mQ_\mB \mA + \mB \mQ_\mA^\top)$ (\textbf{Ours)} &\cellcolor{orange!25} \textcolor{green}{\ding{56}} &\cellcolor{orange!25}\textcolor{green}{\ding{52}}  &\cellcolor{orange!25}\textbf{18.31}{\scriptsize $\pm$0.54}&\cellcolor{orange!25}\textbf{55.21}{\scriptsize $\pm$0.81}&\cellcolor{orange!25}\textbf{34.01}{\scriptsize $\pm$1.50}& 
    \cellcolor{orange!25}\textbf{26.36}{\scriptsize $\pm$0.73}&\cellcolor{orange!25}\textbf{77.13}{\scriptsize $\pm$0.84}&\cellcolor{orange!25}\textbf{47.61}{\scriptsize $\pm$1.67} \\
     \bottomrule
\end{tabular}}
\label{tab:ablation}
\end{table*}

\begin{table}[t!]
\caption{\textbf{Analysis of training overhead.}
We train the LLaMA3-8B model on a single NVIDIA H200 GPU for one epoch.
For SDS-LoRA, we report the training time for the warm-up (first term) and main training phases (second term) separately.
}
\label{tab:overhead}
\centering
\resizebox{0.75\linewidth}{!}{
\begin{tabular}{lcccc}
     \toprule
     \multirow{2}{*}{Method} & \multicolumn{2}{c}{MetaMathQA} & \multicolumn{2}{c}{Code-Feedback}\\
     \cmidrule(lr){2-3} \cmidrule(lr){4-5}
     & Training Time (hours) & GPU-Memory (GB) & Training Time (hours) & GPU-Memory (GB)\\
     \toprule
     \toprule
     LoRA & 1.61 & 128.54 & 1.72 & 128.80\\
     \textbf{\methodname} & 0.01 + 1.67 (\textbf{+4.35\%}) & 129.16 (\textbf{+0.48\%}) & 0.01 + 1.80 (\textbf{+5.23\%}) & 129.77 (\textbf{+0.75\%})\\
     \bottomrule
\end{tabular}}
\vspace{-1em}
\end{table}

\section{Conclusion}
In this paper, we provide a novel perspective on the limitations of Low-Rank Adaptation (LoRA).
Our analysis reveals that LoRA suffers from anisotropic gradient scaling induced by singular values, which distorts the rich information contained in the full fine-tuning gradient and causes LoRA to deviate from full fine-tuning.
To overcome these limitations, we propose a new parameterization for low-rank weight updates that structurally decouples the singular values from the backward pass. Our convergence analysis confirms that the proposed method eliminates the dependence of the convergence rate on the condition number of the LoRA matrices. Experimental results across various benchmarks and pretrained models show that our method significantly accelerates loss convergence and narrows the gap to full fine-tuning, achieving superior adaptation performance.

\paragraph{Limitations.}
In the proposed method, the effective gradient with respect to $\mW_\text{eff}$ deviates slightly from the worst-case-optimal form, as represented by $\mE$ in Equation~\ref{eq:eff_grad_ours}.
This deviation arises from the updates to $\mQ_\mA$ and $\mQ_\mB$.
Although we demonstrate empirically in Section~\ref{sec:subspace_change} that the contribution of this term is small (see also Table~\ref{tab:residual_e}), we do not provide a formal justification.
For theoretical clarity, we assume this gap is negligible in the convergence analysis (Assumption~\ref{assumption:slow_subspace}).
Also, the convergence analysis assumes idealized SGD training scenarios, which differ from the AdamW setting used in our experiments.

\begin{ack}
This work was supported in part by the IITP grants [No. RS-2021-II211343, Artificial Intelligence Graduate School Program (Seoul National University), No. RS-2025-02303870, No. RS-2025-25422680, No.2022-0-00156] funded by the Korean government (MSIT) and the National Research Foundation of Korea (NRF) grant funded by the Korea government (MSIT) (No. RS-2025-24533064).
We also thank Nota Inc. for supporting the revision of this paper and providing travel support for its presentation at NeurIPS 2026.
\end{ack}

\bibliographystyle{plain}
\bibliography{neurips_2026}


\newpage
\appendix

The contents of the Appendix are as follows:
\begin{itemize}[itemsep=5pt]
    \item In Section~\ref{sec:approx_proof}, we provide the proof of Theorem~\ref{theorem:approx}.
    \item In Section~\ref{sec:convergence_proof}, we provide the details and proof of Theorem~\ref{theorem:convergence}.
    \item In Section~\ref{sec:interval_sensitivity}, we provide ablation studies on the update scheduling of $\mQ_\mA$ and $\mQ_\mB$ in our method.
    \item In Section~\ref{sec:warmup_ablation}, we provide ablation studies on the warm-up stage and on the update of $\mQ_\mA$ and $\mQ_\mB$.
    \item In Section~\ref{sec:subspace_change}, we provide empirical evidence showing that the subspace of the LoRA matrices changes slowly in \methodname.
    \item In Section~\ref{sec:stable_rank}, we show that the singular values of the LoRA matrices exhibit a highly skewed spectrum.
    \item In Section~\ref{sec:vision}, we provide the details and experimental results on image classification tasks.
    \item In Section~\ref{sec:precond_discussion}, we provide a detailed discussion on the difference between SDS-LoRA and preconditioning-based methods.
    \item In Section~\ref{sec:orthogonal_lora}, we compare \methodname~with orthogonalization-based LoRA variants.
    \item In Table~\ref{tab:additional_detail}, we provide additional implementation details.
\end{itemize}

\section{Proof of
Theorem~\ref{theorem:approx}}\label{sec:approx_proof}

\begin{proof}
Recall $\mM_\mA = \mV_\mA^\top \mG^\top \mG \mV_\mA$, $\mM_\mB = \mU_\mB^\top \mG \mG^\top \mU_\mB$, $J = \Tr(\mM_\mA \boldsymbol{\Sigma}_\mA^2) + \Tr(\mM_\mB \boldsymbol{\Sigma}_\mB^2)$, and $J^\star(\mG) = E_\mA \lambda_{\max}(\mM_\mA) + E_\mB \lambda_{\max}(\mM_\mB)$, as defined in Section~\ref{sec:lora_limitation}.
We first verify that $\langle \mG, \Tilde{\mG} \rangle_F = s^2 J$.
Using $\langle \mX, \mY \rangle_F = \Tr(\mX^\top \mY)$ and the cyclic property of the trace, the two terms of $\Tilde{\mG}$ in Equation~\ref{eq:eff_grad} give
\begin{equation}\label{eq:approx_trace}
    \begin{split}
        \langle \mG, \mG \mV_\mA \boldsymbol{\Sigma}_\mA^2 \mV_\mA^\top \rangle_F &= \Tr(\mG \mV_\mA \boldsymbol{\Sigma}_\mA^2 \mV_\mA^\top \mG^\top) = \Tr(\mV_\mA^\top \mG^\top \mG \mV_\mA \boldsymbol{\Sigma}_\mA^2) = \Tr(\mM_\mA \boldsymbol{\Sigma}_\mA^2),\\
        \langle \mG, \mU_\mB \boldsymbol{\Sigma}_\mB^2 \mU_\mB^\top \mG \rangle_F &= \Tr(\mG^\top \mU_\mB \boldsymbol{\Sigma}_\mB^2 \mU_\mB^\top \mG) = \Tr(\mU_\mB^\top \mG \mG^\top \mU_\mB \boldsymbol{\Sigma}_\mB^2) = \Tr(\mM_\mB \boldsymbol{\Sigma}_\mB^2).
    \end{split}
\end{equation}
Note that $\mM_\mA$ and $\mM_\mB$ are positive semi-definite by construction, and that the energy constraints are constraints on the spectra, since $\Tr(\boldsymbol{\Sigma}_\mA^2) = ||\mA||_F^2 = E_\mA$ and $\Tr(\boldsymbol{\Sigma}_\mB^2) = ||\mB||_F^2 = E_\mB$.
Throughout, the minimum over $\mG$ is taken over all $\mG$ with $J^\star(\mG) > 0$, that is, over all $\mG$ that have a non-zero component in the given subspaces; otherwise both $J$ and $J^\star(\mG)$ vanish and their ratio is undefined.

\textbf{Step 1: the uniform configuration attains a worst-case ratio of $1/r$.}
For any positive semi-definite $\mX \in \mathbb{R}^{r \times r}$, its trace equals the sum of its $r$ nonnegative eigenvalues, so $\Tr(\mX) \geq \lambda_{\max}(\mX)$, with equality if and only if $\text{rank}(\mX) \leq 1$.
At $\boldsymbol{\Sigma}_\mA^2 = \frac{E_\mA}{r} \mI_r$ and $\boldsymbol{\Sigma}_\mB^2 = \frac{E_\mB}{r} \mI_r$, this gives, for every $\mG$ with $J^\star(\mG) > 0$,
\begin{equation*}
    J = \frac{E_\mA}{r} \Tr(\mM_\mA) + \frac{E_\mB}{r} \Tr(\mM_\mB) \geq \frac{E_\mA}{r} \lambda_{\max}(\mM_\mA) + \frac{E_\mB}{r} \lambda_{\max}(\mM_\mB) = \frac{1}{r} J^\star(\mG),
\end{equation*}
so that $J/J^\star(\mG) \geq \frac{1}{r}$ for every such $\mG$, and hence $\min_{\mG} J/J^\star(\mG) \geq \frac{1}{r}$ at the uniform configuration.
This bound is attained.
Let $\hat{\vp}$ be any unit vector in the row space of $\mA$ and $\hat{\vu}$ any unit vector in the column space of $\mB$, and set $\mG = \hat{\vu}\hat{\vp}^\top$.
Then $\mG^\top \mG = \hat{\vp}\hat{\vp}^\top$ and $\mG \mG^\top = \hat{\vu}\hat{\vu}^\top$, so that $\mM_\mA = (\mV_\mA^\top \hat{\vp})(\mV_\mA^\top \hat{\vp})^\top$ and $\mM_\mB = (\mU_\mB^\top \hat{\vu})(\mU_\mB^\top \hat{\vu})^\top$ have rank one for any such $\hat{\vp}$ and $\hat{\vu}$.
Consequently $\Tr(\mM_\mA) = \lambda_{\max}(\mM_\mA)$ and $\Tr(\mM_\mB) = \lambda_{\max}(\mM_\mB)$, and the inequality above holds with equality.
Hence $\min_{\mG} J/J^\star(\mG) = \frac{1}{r}$ exactly at the uniform configuration.

\textbf{Step 2: no configuration attains a worst-case ratio exceeding $1/r$, and only the uniform configuration attains it.}
Fix any $\boldsymbol{\Sigma}_\mA^2$ and $\boldsymbol{\Sigma}_\mB^2$ satisfying $\Tr(\boldsymbol{\Sigma}_\mA^2) = E_\mA$ and $\Tr(\boldsymbol{\Sigma}_\mB^2) = E_\mB$.
Let $i_{\min}$ and $j_{\min}$ index their smallest diagonal entries, let $\hat{\vp}$ be the $i_{\min}$-th column of $\mV_\mA$ and $\hat{\vu}$ the $j_{\min}$-th column of $\mU_\mB$, and consider the adversarial gradient $\mG = \hat{\vu}\hat{\vp}^\top$.
Since $\mV_\mA^\top \hat{\vp} = \ve_{i_{\min}}$ and $\mU_\mB^\top \hat{\vu} = \ve_{j_{\min}}$, the columns of $\mV_\mA$ and $\mU_\mB$ being orthonormal, the computation in Step 1 gives $\mM_\mA = \ve_{i_{\min}} \ve_{i_{\min}}^\top$ and $\mM_\mB = \ve_{j_{\min}} \ve_{j_{\min}}^\top$, so that $\lambda_{\max}(\mM_\mA) = \lambda_{\max}(\mM_\mB) = 1$ and $J^\star(\mG) = E_\mA + E_\mB$, while
\begin{equation*}
    J = \Tr(\ve_{i_{\min}} \ve_{i_{\min}}^\top \boldsymbol{\Sigma}_\mA^2) + \Tr(\ve_{j_{\min}} \ve_{j_{\min}}^\top \boldsymbol{\Sigma}_\mB^2) = \lambda_{\min}(\boldsymbol{\Sigma}_\mA^2) + \lambda_{\min}(\boldsymbol{\Sigma}_\mB^2).
\end{equation*}
Thus, for this particular $\mG$,
\begin{equation}\label{eq:approx_worstcase}
    \frac{J}{J^\star(\mG)} = \frac{\lambda_{\min}(\boldsymbol{\Sigma}_\mA^2) + \lambda_{\min}(\boldsymbol{\Sigma}_\mB^2)}{E_\mA + E_\mB}.
\end{equation}
The minimum eigenvalue of an $r \times r$ positive semi-definite matrix never exceeds its average eigenvalue, so $\lambda_{\min}(\boldsymbol{\Sigma}_\mA^2) \leq \frac{E_\mA}{r}$ and $\lambda_{\min}(\boldsymbol{\Sigma}_\mB^2) \leq \frac{E_\mB}{r}$, each with equality if and only if the corresponding spectrum is uniform.
Hence $J/J^\star(\mG) \leq \frac{1}{r}$ at this $\mG$, with strict inequality unless both $\boldsymbol{\Sigma}_\mA^2$ and $\boldsymbol{\Sigma}_\mB^2$ are uniform.
Since $\min_{\mG'} J/J^\star(\mG') \leq J/J^\star(\mG)$ for this particular $\mG$, we conclude that $\min_{\mG'} J/J^\star(\mG') \leq \frac{1}{r}$ for every configuration, strictly unless both $\boldsymbol{\Sigma}_\mA^2$ and $\boldsymbol{\Sigma}_\mB^2$ are uniform.

Combining Steps 1 and 2, $\max_{\boldsymbol{\Sigma}_\mA^2, \boldsymbol{\Sigma}_\mB^2} \min_{\mG} J/J^\star(\mG) = \frac{1}{r}$, uniquely attained at $\boldsymbol{\Sigma}_\mA = \sqrt{E_\mA/r} \mI_r$ and $\boldsymbol{\Sigma}_\mB = \sqrt{E_\mB/r} \mI_r$.
\end{proof}

\section{Details and Proof of Theorem~\ref{theorem:convergence}}\label{sec:convergence_proof}
In this section, we provide details and proof of Theorem~\ref{theorem:convergence}.
We first formalize the assumptions and definitions for the theorem.

\begin{assumption}\label{assumption:beta_smooth}
    (\textbf{$\beta$-smoothness})
    The loss function $\mathcal{L}$ is $\beta$-smooth with respect to $\mW_{\normalfont\text{eff}}$:
    \begin{equation}\label{eq:beta_smooth}
        \gL(\mW_{\normalfont\text{eff}}+\Delta\mW_{\normalfont\text{eff}}) \leq \gL(\mW_{\normalfont\text{eff}}) + \underbrace{\langle \mG, \Delta\mW_{\normalfont\text{eff}} \rangle_F}_{\text{Descent}} + \underbrace{\frac{\beta}{2} ||\Delta\mW_{\normalfont\text{eff}}||_F^2}_{\text{Penalty}}.
    \end{equation}
\end{assumption}

\begin{assumption}\label{assumption:pl_condition}
    (\textbf{Local PL condition})
    The loss function $\mathcal{L}$ satisfies the $\mu$-Polyak–Łojasiewicz (PL) condition near the pretrained weight $\mW_0$: $||\nabla_{\mW_{\normalfont\text{eff}}} \mathcal{L}||_F^2 \geq 2\mu(\mathcal{L}(\mW_{\normalfont\text{eff}})-\mathcal{L}^*)$ for $\|\mW_{\text{eff}}-\mW_0\|_F \leq \epsilon$, where $\gL^*$ denotes the minimal loss.
    We assume that $\|\mW_{\text{eff}}-\mW_0\|_F \leq \epsilon$ remains satisfied throughout training.
\end{assumption}

\begin{assumption}\label{assumption:slow_subspace}
    (\textbf{Slow subspace change in \methodname})
    In \methodname, the change in $\mQ_\mA$ and $\mQ_\mB$ is negligible after each update of $\mA$ and $\mB$ (i.e., $\Delta \mQ_\mA \approx 0$ and $\Delta \mQ_\mB \approx 0$).
\end{assumption}
See Section~\ref{sec:subspace_change} for empirical validation of Assumption~\ref{assumption:slow_subspace}.

\begin{definition}
    Let $\alpha \in [0,2]$ represent the alignment between $\mG$ and the low-rank subspaces $\gR(\mA)$ and $\gC(\mB)$ such that
    \begin{equation}\label{eq:alpha_alignment}
        ||\mG \mP_{\gR(\mA)}||_F^2+||\mP_{\gC(\mB)}\mG||_F^2 = \alpha ||\mG||_F^2.
    \end{equation}
\end{definition}

In Equation~\ref{eq:beta_smooth}, two terms determine the loss decrease bound: the descent term and the penalty term.
We first derive lemmas on these terms.

\begin{lemma}\label{lemma:descent}
    (\textbf{Descent term})
    For LoRA, the descent term in Equation~\ref{eq:beta_smooth} is bounded as follows:
    \begin{equation}
        \langle \mG, \Delta \mW_{\normalfont\text{eff}} \rangle_F \leq -\alpha \eta s^2 \sigma_{\min}^2(\mA,\mB) ||\mG||_F^2 \:\:\:\:\:\:\text{(in LoRA)},
    \end{equation}
    where $\sigma_{\min}(\mA,\mB) = \min(\sigma_{\min}(\mA), \sigma_{\min}(\mB))$.
    For \methodname, under Assumption~\ref{assumption:slow_subspace}, the descent term is given by
    \begin{equation}
        \langle \mG, \Delta \mW_{\normalfont\text{eff}} \rangle_F \approx -\alpha \eta s^2 ||\mG||_F^2 \:\:\:\:\:\:\text{(in SDS-LoRA)}.
    \end{equation}
\end{lemma}
\begin{proof}
    \textbf{1. LoRA}

    For LoRA, the gradients with respect to $\mA$ and $\mB$ are given by: $\nabla_\mA \gL = s \mB^\top \mG$ and $\nabla_\mB \gL = s \mG \mA^\top$.
    Consequently, the effective weight update is derived as:
    $\Delta \mW_{\text{eff}} \approx s(\mB \Delta \mA + \Delta \mB \mA) = s(\mB (-\eta s \mB^\top \mG) + (-\eta s \mG \mA^\top) \mA) = -\eta s^2 (\mB \mB^\top \mG + \mG \mA^\top \mA)$.
    Substituting this into the descent term, we obtain
    \begin{equation}
        \begin{split}
            \langle \mG, \Delta \mW_{\normalfont\text{eff}} \rangle_F &= \langle \mG, -\eta s^2 (\mB \mB^\top \mG + \mG \mA^\top \mA) \rangle_F \\
            &= -\eta s^2 ( \langle \mG, \mB \mB^\top \mG \rangle_F + \langle \mG, \mG \mA^\top \mA \rangle_F ) \\
            &= -\eta s^2 ( ||\mB^\top \mG||_F^2 + ||\mG \mA^\top||_F^2 ).
        \end{split}
    \end{equation}
    We relate these terms to the projections onto the subspaces $\gC(\mB)$ and $\gR(\mA)$. For any matrix $\mX$ and $\mB$, $||\mB^\top \mX||_F^2 = ||\mB^\top \mP_{\gC(\mB)} \mX||_F^2 \geq \sigma_{\min}^2(\mB) ||\mP_{\gC(\mB)} \mX||_F^2$.
    Applying this inequality:
    \begin{equation}
        \begin{split}
            ||\mB^\top \mG||_F^2 &\geq \sigma_{\min}^2(\mB) ||\mP_{\gC(\mB)} \mG||_F^2, \\
            ||\mG \mA^\top||_F^2 &\geq \sigma_{\min}^2(\mA) ||\mG \mP_{\gR(\mA)}||_F^2.
        \end{split}
    \end{equation}
    Thus, using $\sigma_{\min}^2(\mA,\mB) = \min(\sigma_{\min}^2(\mA), \sigma_{\min}^2(\mB))$ and Equation~\ref{eq:alpha_alignment}, we obtain
    \begin{equation}
        \begin{split}
            \langle \mG, \Delta \mW_{\normalfont\text{eff}} \rangle_F &\leq -\eta s^2 \sigma_{\min}^2(\mA,\mB) ( ||\mP_{\gC(\mB)} \mG||_F^2 + ||\mG \mP_{\gR(\mA)}||_F^2 ) \\
            &= -\alpha \eta s^2 \sigma_{\min}^2(\mA,\mB) ||\mG||_F^2.
        \end{split}
    \end{equation}
    
    \textbf{2. \methodname}

    Under Assumption~\ref{assumption:slow_subspace}, the effective weight update is given by:
    $\Delta \mW_\text{eff} \approx s(\mQ_\mB \Delta \mA + \Delta \mB \mQ_\mA^\top + \Delta \mQ_\mB \mA + \mB \Delta \mQ_\mA^\top) \approx -\eta s^2 (\mP_{\gC(\mB)} \mG + \mG \mP_{\gR(\mA)})$.
    Substituting this into the inner product:
    \begin{equation}
        \begin{split}
            \langle \mG, \Delta \mW_{\normalfont\text{eff}} \rangle_F & \approx \langle \mG, -\eta s^2 (\mP_{\gC(\mB)} \mG + \mG \mP_{\gR(\mA)}) \rangle_F \\
            &=-\eta s^2 ( \langle \mG, \mP_{\gC(\mB)} \mG \rangle_F + \langle \mG, \mG \mP_{\gR(\mA)} \rangle_F).
        \end{split}
    \end{equation}
    For any matrix $\mX$ and projection $\mP$, $\langle \mX, \mP \mX \rangle_F = \Tr(\mX^\top \mP \mX) = \Tr(\mX^\top \mP^\top \mP \mX) = ||\mP \mX||_F^2$.
    Applying this to our terms:
    \begin{equation}
        \begin{split}
            \langle \mG, \mP_{\gC(\mB)} \mG \rangle_F &= ||\mP_{\gC(\mB)} \mG||_F^2, \\
            \langle \mG, \mG \mP_{\gR(\mA)} \rangle_F &= \langle \mG^\top, \mP_{\gR(\mA)}^\top \mG^\top \rangle_F = ||\mG \mP_{\gR(\mA)}||_F^2.
        \end{split}
    \end{equation}
    Finally, applying Equation~\ref{eq:alpha_alignment}, we obtain
    \begin{equation}
        \begin{split}
            \langle \mG, \Delta \mW_{\normalfont\text{eff}} \rangle_F &= -\eta s^2 (||\mG \mP_{\gR(\mA)}||_F^2+||\mP_{\gC(\mB)}\mG||_F^2)\\
            &= -\alpha \eta s^2 ||\mG||_F^2.
        \end{split}
    \end{equation}
\end{proof}

\begin{lemma}\label{lemma:penalty}
    (\textbf{Penalty term})
    For LoRA, the bound for the penalty term in Equation~\ref{eq:beta_smooth} is given by
    \begin{equation}
        \text{Penalty}_\text{LoRA} \leq \alpha \beta \eta^2 s^4 \sigma_{\normalfont\text{max}}^4(\mA,\mB) ||\mG||_F^2 \:\:\:\:\:\:\text{(in LoRA)}.
    \end{equation}
    For \methodname, the bound for the penalty term is given by
    \begin{equation}
        \text{Penalty}_\text{\methodname} \leq \alpha \beta \eta^2 s^4 ||\mG||_F^2 \:\:\:\:\:\:\text{(in SDS-LoRA)}.
    \end{equation}
    where $\sigma_{\normalfont\text{max}}(\mA,\mB) = \max(\sigma_{\normalfont\text{max}}(\mA),\sigma_{\normalfont\text{max}}(\mB))$.
\end{lemma}
\begin{proof}
    We bound the penalty term $\frac{\beta}{2} ||\Delta \mW_\text{eff}||_F^2$ by relating it to the updates of $\mA$ and $\mB$.
    Both parameterizations induce a weight change of the common form $\Delta \mW_\text{eff} \approx s (\mX \Delta \mA + \Delta \mB \mY)$, where $(\mX, \mY) = (\mB, \mA)$ for LoRA, for which the second-order term $s \Delta \mB \Delta \mA$ is dropped, and $(\mX, \mY) = (\mQ_\mB, \mQ_\mA^\top)$ for \methodname, for which the form is exact under Assumption~\ref{assumption:slow_subspace}.
    The triangle inequality and sub-multiplicative property of norms give us
    \begin{equation}
        \begin{split}
            ||\Delta \mW_\text{eff}||_F &\approx s ||\mX \Delta \mA + \Delta \mB \mY||_F \\
            &\leq s (||\mX||_2 ||\Delta \mA||_F + ||\Delta \mB||_F ||\mY||_2) \\
            &\leq s \rho (||\Delta \mA||_F + ||\Delta \mB||_F),
        \end{split}
    \end{equation}
    where $\rho = \max(||\mX||_2, ||\mY||_2)$, so that $\rho = \sigma_{\text{max}}(\mA, \mB)$ for LoRA and $\rho = 1$ for \methodname.
    Using $(x+y)^2 \leq 2x^2 + 2y^2$, we obtain
    \begin{equation}\label{eq:smoothness_param}
        ||\Delta \mW_\text{eff}||_F^2 \leq 2 s^2 \rho^2 (||\Delta \mA||_F^2 + ||\Delta \mB||_F^2).
    \end{equation}
    
    \textbf{1. LoRA}
    
    We relate the norm of $\Delta \mA = -\eta s \mB^\top \mG$ and $\Delta \mB = -\eta s \mG \mA^\top$ to the subspace projections:
    \begin{equation}
        \begin{split}
            ||\Delta \mA||_F^2 &= \eta^2 s^2 ||\mB^\top \mG||_F^2 = \eta^2 s^2 ||\mB^\top \mP_{\gC(\mB)} \mG||_F^2 \leq \eta^2 s^2 \sigma_{\text{max}}^2(\mB) ||\mP_{\gC(\mB)} \mG||_F^2, \\
            ||\Delta \mB||_F^2 &= \eta^2 s^2 ||\mG \mA^\top||_F^2 = \eta^2 s^2 ||\mG \mP_{\gR(\mA)} \mA^\top||_F^2 \leq \eta^2 s^2 \sigma_{\text{max}}^2(\mA) ||\mG \mP_{\gR(\mA)}||_F^2.
        \end{split}
    \end{equation}
    Substituting these into Equation~\ref{eq:smoothness_param} with $\rho = \sigma_{\text{max}}(\mA, \mB)$:
    \begin{equation}
        \begin{split}
            ||\Delta \mW_\text{eff}||_F^2 &\leq 2 s^2 \sigma_{\text{max}}^2(\mA, \mB) \left( \eta^2 s^2 \sigma_{\text{max}}^2(\mB) ||\mP_{\gC(\mB)} \mG||_F^2 + \eta^2 s^2 \sigma_{\text{max}}^2(\mA) ||\mG \mP_{\gR(\mA)}||_F^2 \right) \\
            &\leq 2 \eta^2 s^4 \sigma_{\text{max}}^4(\mA, \mB) \left( ||\mP_{\gC(\mB)} \mG||_F^2 + ||\mG \mP_{\gR(\mA)}||_F^2 \right).
        \end{split}
    \end{equation}
    Applying Equation~\ref{eq:alpha_alignment}:
    \begin{equation}
        ||\Delta \mW_\text{eff}||_F^2 \leq 2 \alpha \eta^2 s^4 \sigma_{\text{max}}^4(\mA, \mB) ||\mG||_F^2.
    \end{equation}
    The penalty is therefore $\frac{\beta}{2} ||\Delta \mW_\text{eff}||_F^2 \leq \alpha \beta \eta^2 s^4 \sigma_{\text{max}}^4(\mA, \mB) ||\mG||_F^2$.

    \textbf{2. \methodname}
    
    Since $\mQ_\mB$ and $\mQ_\mA^\top$ have orthonormal columns and rows, respectively, their spectral norms are 1. Thus, $\rho = 1$.
    Using the update rules $\Delta \mA = -\eta s \mQ_\mB^\top \mG$ and $\Delta \mB = -\eta s \mG \mQ_\mA$:
    \begin{equation}
        \begin{split}
            ||\Delta \mA||_F^2 &= \eta^2 s^2 ||\mQ_\mB^\top \mG||_F^2 = \eta^2 s^2 \Tr(\mG^\top \mQ_\mB \mQ_\mB^\top \mG) = \eta^2 s^2 ||\mP_{\gC(\mB)} \mG||_F^2. \\
            ||\Delta \mB||_F^2 &= \eta^2 s^2 ||\mG \mQ_\mA||_F^2 = \eta^2 s^2 \Tr(\mG \mQ_\mA \mQ_\mA^\top \mG^\top) = \eta^2 s^2 ||\mG \mP_{\gR(\mA)}||_F^2.
        \end{split}
    \end{equation}
    Substituting these into Equation~\ref{eq:smoothness_param} with $\rho = 1$:
    \begin{equation}
        \begin{split}
            ||\Delta \mW_\text{eff}||_F^2 &\leq 2 s^2 \cdot 1^2 \cdot \eta^2 s^2 (||\mP_{\gC(\mB)} \mG||_F^2 + ||\mG \mP_{\gR(\mA)}||_F^2).
        \end{split}
    \end{equation}
    Applying Equation~\ref{eq:alpha_alignment}:
    \begin{equation}
        ||\Delta \mW_\text{eff}||_F^2 \leq 2 \alpha \eta^2 s^4 ||\mG||_F^2.
    \end{equation}
    The penalty is therefore $\frac{\beta}{2} ||\Delta \mW_\text{eff}||_F^2 \leq \alpha \beta \eta^2 s^4 ||\mG||_F^2$.
\end{proof}

Finally, using Lemma~\ref{lemma:descent} and Lemma~\ref{lemma:penalty}, we prove Theorem~\ref{theorem:convergence}.
\begin{proof}
    Let $\gL_t = \gL(\mW_\text{eff})$ and $\gL_{t+1} = \gL(\mW_\text{eff} + \Delta \mW_\text{eff})$.

    \textbf{1. LoRA}

    Substitute the Descent and Penalty lemmas for LoRA:
    \begin{equation}\label{eq:reform_lora}
        \gL_{t+1} \leq \gL_t - \alpha \eta s^2 \sigma_{\min}^2 ||\mG||_F^2 + \alpha \beta \eta^2 s^4 \sigma_{\max}^4 ||\mG||_F^2.
    \end{equation}
    We set $f(\eta) = -\alpha s^2 \sigma_{\min}^2 \eta + \alpha \beta s^4 \sigma_{\max}^4 \eta^2$. Solving for optimal $\eta^*$:
    \begin{equation}
        f^\prime(\eta^*) = -\alpha s^2 \sigma_{\min}^2 + 2 \alpha \beta s^4 \sigma_{\max}^4 \eta = 0 \implies \eta^* = \frac{s^2 \sigma_{\min}^2}{2 \beta s^4 \sigma_{\max}^4} = \frac{\sigma_{\min}^2}{2 \beta s^2 \sigma_{\max}^4}.
    \end{equation}
    Substituting $\eta^*$ back into Equation~\ref{eq:reform_lora}:
    \begin{equation}
        \begin{split}
            \gL_{t+1} &\leq \gL_t - \alpha s^2 \sigma_{\min}^2 \left(\frac{\sigma_{\min}^2}{2 \beta s^2 \sigma_{\max}^4}\right) ||\mG||_F^2 + \alpha \beta s^4 \sigma_{\max}^4 \left(\frac{\sigma_{\min}^2}{2 \beta s^2 \sigma_{\max}^4}\right)^2 ||\mG||_F^2 \\
            &= \gL_t - \frac{\alpha \sigma_{\min}^4}{2 \beta \sigma_{\max}^4} ||\mG||_F^2 + \frac{\alpha \sigma_{\min}^4}{4 \beta \sigma_{\max}^4} ||\mG||_F^2 \\
            &= \gL_t - \frac{\alpha \sigma_{\min}^4}{4 \beta \sigma_{\max}^4} ||\mG||_F^2.
        \end{split}
    \end{equation}
    Recognizing that $\frac{\sigma_{\max}^4}{\sigma_{\min}^4} = \kappa^4$ for $\kappa = \frac{\sigma_{\normalfont\text{max}}(\mA,\mB)}{\sigma_{\normalfont\text{min}}(\mA,\mB)}$, we rewrite this as $\gL_{t+1} \leq \gL_t - \frac{\alpha}{4 \beta \kappa^4} ||\mG||_F^2$.
    Applying the PL condition:
    \begin{equation}
        \begin{split}
            \gL_{t+1} &\leq \gL_t - \frac{\alpha}{4\beta\kappa^4}2\mu(\gL_t - \gL^*)\\
            \gL_{t+1} - \gL^*&\leq (1-\frac{\mu\alpha}{2\beta\kappa^4})(\gL_t-\gL^*).
        \end{split}
    \end{equation}

    \textbf{2. \methodname}
    
    Substitute the Descent and Penalty lemmas into Equation~\ref{eq:beta_smooth}:
    \begin{equation}\label{eq:reform_sdslora}
        \gL_{t+1} \leq \gL_t - \alpha \eta s^2 ||\mG||_F^2 + \alpha \beta \eta^2 s^4 ||\mG||_F^2.
    \end{equation}
    To find the optimal step size, we define $f(\eta) = -\alpha s^2 \eta + \alpha \beta s^4 \eta^2$ and solve $f^\prime(\eta^*) = -\alpha s^2 + 2\alpha\beta s^4 \eta = 0$.
    This gives us $\eta^* = \frac{s^2}{2\beta s^4} = \frac{1}{2\beta s^2}$.
    By substituting $\eta^*$ back into Equation~\ref{eq:reform_sdslora}, we obtain:
    \begin{equation}
        \gL_{t+1} \leq \gL_t - \frac{\alpha s^2}{2\beta s^2}||\mG||_F^2 + \frac{\alpha \beta s^4}{4\beta^2 s^4}||\mG||_F^2 = \gL_t - \frac{\alpha}{4\beta}||\mG||_F^2.
    \end{equation}
    Applying the PL condition, we obtain
    \begin{equation}
        \begin{split}
            \gL_{t+1} &\leq \gL_t- \frac{\alpha}{4\beta}2\mu(\gL_t - \gL^*)\\
            \gL_{t+1} - \gL^*& \leq (1-\frac{\mu\alpha}{2\beta})(\gL_t-\gL^*).
        \end{split}
    \end{equation}

\end{proof}

\section{Ablation Studies on the Update Interval}\label{sec:interval_sensitivity}
In this section, we provide ablation studies on the update interval for $\mQ_\mA$ and $\mQ_\mB$ in the training procedure of \methodname~described in Algorithm~\ref{alg:training}.
To examine the effectiveness of the proposed interval scheduling (i.e., $t \bmod \left\lceil \frac{kt}{T_\text{total}} \right\rceil = 0$), we compare it against a uniform schedule (i.e., $t \bmod k = 0$).
Table~\ref{tab:interval_sensitivity} demonstrates that for both schedules, performance degrades as the update interval increases.
This occurs because large update intervals cause the updates of $\mQ_{\mA}$ and $\mQ_{\mB}$ to significantly disrupt the model, resulting in unstable training.
We also observe that frequent updates do not always yield better performance, justifying the use of intermittent updates for efficiency. 
Furthermore, given the same total number of updates, our scheduling method outperforms the uniform schedule, demonstrating the necessity of frequent updates during the initial training phase.

\begin{table*}[b!]
\caption{\textbf{Ablation studies on the update interval of $\mQ_\mA$ and $\mQ_\mB$.}
Experiments are conducted on LLaMA3-8B with rank 32.
}
\vspace{-0.5em}
\label{tab:interval_sensitivity}
\centering
\resizebox{0.6\linewidth}{!}{
\begin{tabular}{lcccc}
     \toprule
     \multicolumn{2}{c}{\multirow{2}{*}{Update Condition}} & Number of Updates & \multicolumn{2}{c}{LLaMA3-8B}\\
     \cmidrule(lr){4-5}
     &&($\times T_\text{Total}$)& MATH & GSM8K\\
     \toprule
     \toprule
     \multirow{3}{*}{$t \bmod k = 0$} & $k=1$ & 1 & 26.42{\scriptsize $\pm$0.76}&76.91{\scriptsize $\pm$0.57} \\
     &$k=2$& 0.5 & 26.24{\scriptsize $\pm$0.81}&77.01{\scriptsize $\pm$0.73}  \\
     &$k=3$& 0.33 & 24.58{\scriptsize $\pm$0.42}&75.50{\scriptsize $\pm$0.74}  \\
     \midrule
     \multirow{5}{*}{$t \bmod \left\lceil \frac{kt}{T_\text{total}} \right\rceil = 0$}& $k=3$& 0.61 & 26.53{\scriptsize $\pm$0.66}&77.04{\scriptsize $\pm$0.72}\\
     & $k=4$& 0.52 & 26.42{\scriptsize $\pm$0.54}&77.16{\scriptsize $\pm$0.68}\\
     & $k=5$ (Default)& 0.46 & 26.36{\scriptsize$\pm$0.73} & 77.13{\scriptsize$\pm$0.84} \\
     & $k=6$& 0.41 & 25.88{\scriptsize $\pm$0.44}&76.65{\scriptsize $\pm$0.52}\\
     & $k=7$& 0.37 & 25.03{\scriptsize $\pm$0.48}&76.21{\scriptsize $\pm$0.78}\\
     \bottomrule
\end{tabular}}
\vspace{-0.5em}
\end{table*}

\section{Ablation Studies on the Warm-Up Stage and the Update of Orthonormal Bases}\label{sec:warmup_ablation}

\paragraph{Warm-up stage.}
The warm-up stage serves as a lightweight enabler that defines $\mQ_\mA$ and $\mQ_\mB$ for the \methodname~formulation, as described in Equation~\ref{eq:initialization}.
To examine whether the gain of \methodname~originates from this stage, we ablate both the initialization scheme and the warm-up length in Table~\ref{tab:warmup_ablation}.
Here, `Random' denotes an initialization in which $\mA$ and $\mB$ are randomly initialized and the orthonormal bases are computed from them, after which the resulting $\Delta \mW$ is subtracted from the pretrained weight so that training still starts from $\mW_0$.
Random initialization degrades performance regardless of the warm-up length, because a randomly defined subspace does not align well with the full fine-tuning gradient, whereas the warm-up stage defines the subspace from actual training examples.
A longer warm-up does not help either and slightly degrades performance, which we attribute to the increasing influence of the standard LoRA formulation.
Together with the `\methodname~w/o detachment' variant in Table~\ref{tab:orthogonal_lora}, which shares the same warm-up stage but performs substantially worse, these results indicate that the warm-up stage is a lightweight enabler for the \methodname~parameterization rather than a component to be tuned for performance.

\paragraph{Update of $\mQ_\mA$ and $\mQ_\mB$.}
We also examine what happens when $\mQ_\mA$ and $\mQ_\mB$ are never updated after the warm-up stage.
Table~\ref{tab:q_frozen} shows that performance drops substantially in this case, since freezing $\mQ_\mA$ and $\mQ_\mB$ confines the update to a fixed subspace that can no longer effectively capture the full gradient.
As reported in Table~\ref{tab:overhead}, updating them incurs approximately 5\% time and less than 1\% memory overhead, which is small relative to the resulting performance gain.

\begin{table}[t!]
\caption{\textbf{Ablation studies on the initialization scheme and the warm-up length.}
We report the average accuracy on the commonsense reasoning benchmarks with Gemma-2B at rank 32 from a single run.
}
\label{tab:warmup_ablation}
\centering
\resizebox{0.5\linewidth}{!}{
\begin{tabular}{lcc}
     \toprule
     Initialization & Warm-Up Iterations $T$ & Avg. Accuracy\\
     \toprule
     \toprule
     Random & 0 & 71.24\\
     Random & 10 & 71.16\\
     \midrule
     SVD (ours) & 10 (default) & \textbf{74.80}\\
     SVD (ours) & 50 & 74.67\\
     SVD (ours) & 100 & 74.54\\
     \bottomrule
\end{tabular}}
\vspace{-0.5em}
\end{table}

\begin{table}[t!]
\caption{\textbf{Ablation study on the update of $\mQ_\mA$ and $\mQ_\mB$.}
We report the average accuracy on the commonsense reasoning benchmarks with Gemma-2B at rank 32 from a single run.
}
\label{tab:q_frozen}
\centering
\resizebox{0.45\linewidth}{!}{
\begin{tabular}{lc}
     \toprule
     Update of $\mQ_\mA$ and $\mQ_\mB$ & Avg. Accuracy\\
     \toprule
     \toprule
     Never after warm-up & 64.78\\
     Periodic (ours) & \textbf{74.80}\\
     \bottomrule
\end{tabular}}
\vspace{-0.5em}
\end{table}

\section{Subspace Change in \methodname}\label{sec:subspace_change}
In Equation~\ref{eq:eff_grad_ours}, the term $\mE$, which arises from the change in $\mQ_\mA$ and $\mQ_\mB$, causes the effective gradient to deviate from the optimal effective gradient.
To assess the contribution of $\mE$, we examine the rate at which $\mQ_\mA$ and $\mQ_\mB$, which represent the subspaces of $\mA$ and $\mB$, evolve during training.
Specifically, we measure the average cosine similarity between $\mQ_\mA$ and $\mQ_\mA + \Delta \mQ_\mA$ (or $\mQ_\mB$ and $\mQ_\mB + \Delta \mQ_\mB$) across all layers.
Figure~\ref{fig:subspace_change} shows that the cosine similarity remains close to 1 throughout training, except for a few initial iterations.
This suggests that the error arising from the subspace change is negligible.
Thus, \methodname~nearly achieves the worst-case-optimal gradient quality characterized in Theorem~\ref{theorem:approx}, providing empirical support for Assumption~\ref{assumption:slow_subspace} in our convergence analysis.

To assess the contribution of $\mE$ more directly, we additionally measure its relative magnitude within the effective gradient of Equation~\ref{eq:eff_grad_ours}, that is, the ratio $||\mE||_F / ||\Tilde{\mG}||_F$.
Since $\Delta \mW_\text{eff} = -\eta \Tilde{\mG}$, this ratio is exactly the fraction of the effective weight update that originates from the change in $\mQ_\mA$ and $\mQ_\mB$, as the common factor cancels between the numerator and the denominator.
Table~\ref{tab:residual_e} shows that the contribution of $\mE$ stays below $0.1$ throughout training and decreases further as training progresses.
We acknowledge that this may still be insufficient to strictly justify $\mE \approx 0$, and \methodname~is therefore best regarded as a practical realization of our theoretical formulation.

\begin{table}[t!]
\caption{\textbf{Contribution of the residual term $\mE$ to the effective gradient.}
We divide training into three phases and report the ratio averaged over all layers within each phase.
We fine-tune Gemma-2B using the Commonsense-170K dataset with rank 32.
}
\label{tab:residual_e}
\centering
\resizebox{0.6\linewidth}{!}{
\begin{tabular}{lccc}
     \toprule
     Training progress & Early & Mid & Late\\
     \toprule
     \toprule
     $||\mE||_F \: / \: ||s^2(\mG \mQ_\mA \mQ_\mA^\top + \mQ_\mB \mQ_\mB^\top \mG) + \mE||_F$ & 0.09 & 0.05 & 0.03\\
     \bottomrule
\end{tabular}}
\vspace{-0.5em}
\end{table}

\begin{figure}[b!]
    \centering
    \begin{subfigure}[b]{0.4\textwidth}
        \centering
        \includegraphics[width=\textwidth]{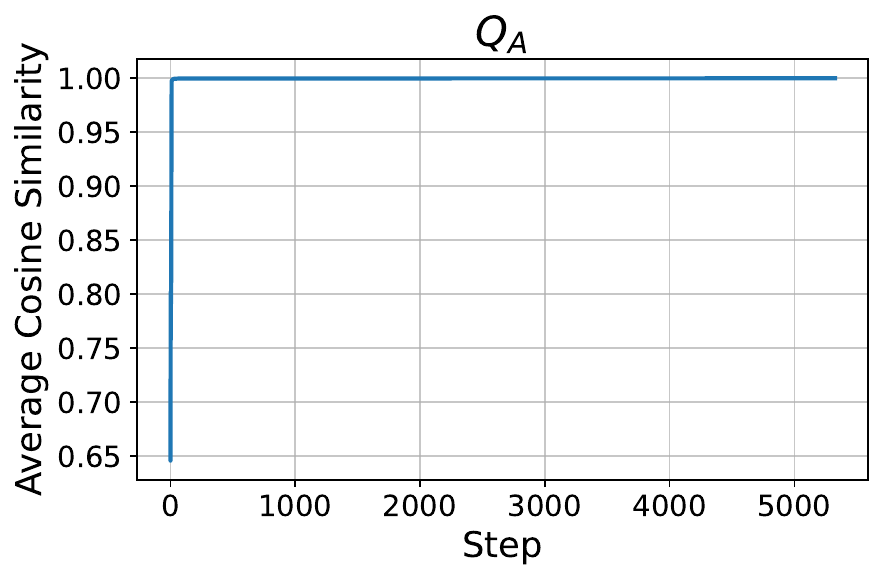}
        \vspace{-0.1em}
    \end{subfigure}
    \begin{subfigure}[b]{0.4\textwidth}
        \centering
        \includegraphics[width=\textwidth]{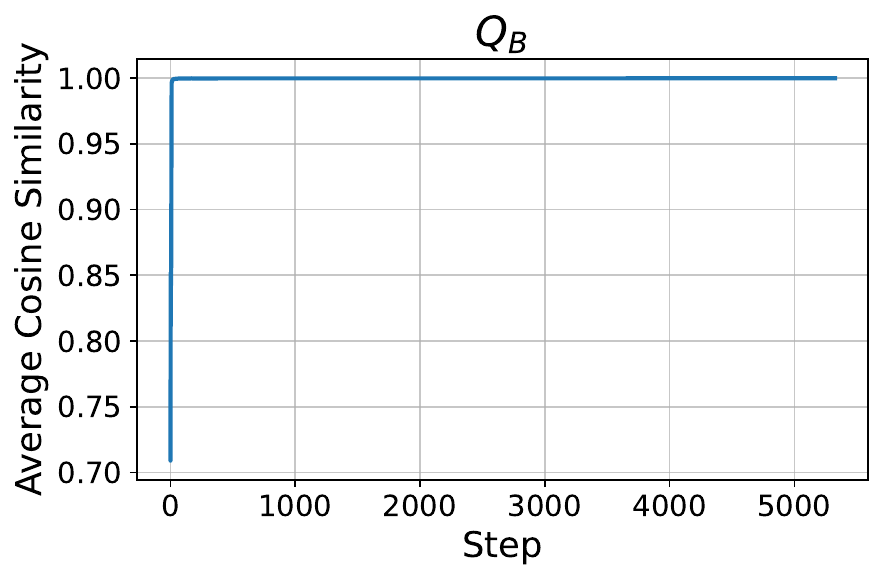}
        \vspace{-0.1em}
    \end{subfigure}
    \vspace{-2em}
    \caption{\textbf{Rate of change in $\mQ_\mA$ and $\mQ_\mB$ during training.}
    During training with \methodname, we monitor the average cosine similarity between $\mQ_\mA$ and $\mQ_\mA +\Delta \mQ_\mA$ and between $\mQ_\mB$ and $\mQ_\mB +\Delta \mQ_\mB$ across all layers.
    We fine-tune LLaMA3-8B using the Commonsense-170K dataset with rank 32.
    }
    \label{fig:subspace_change}
\end{figure}
\section{Singular Value Distribution of LoRA Matrices}\label{sec:stable_rank}

Our claim regarding the adverse effects of anisotropic gradient scaling, which is induced by the singular values of the LoRA matrices, holds only when the singular values exhibit a highly skewed spectrum.
To verify this, we monitor the stable rank of $\mA$ and $\mB$ throughout the LoRA training process, defined as $\frac{\|\mA\|_F^2}{\|\mA\|_2^2}$ and $\frac{\|\mB\|_F^2}{\|\mB\|_2^2}$, where $\|\cdot\|_2$ denotes the spectral norm.
Figure~\ref{fig:stable_rank} demonstrates that the average stable rank for both matrices is significantly lower than the LoRA rank (i.e., 32).
This observation confirms that the singular values are indeed highly skewed, thereby validating our claim.

\begin{figure}[t!]
    \centering
    \begin{subfigure}[b]{0.4\textwidth}
        \centering
        \includegraphics[width=\textwidth]{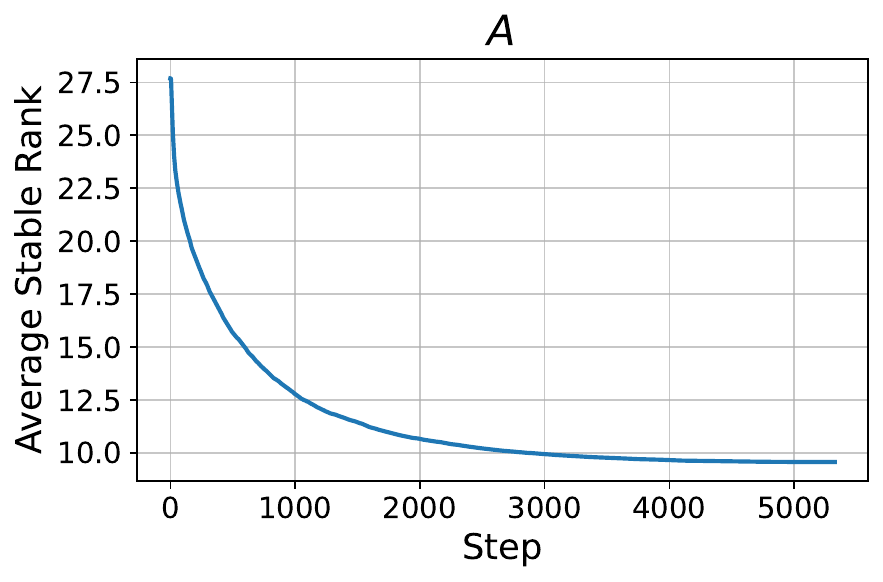}
        \vspace{-0.1em}
    \end{subfigure}
    \begin{subfigure}[b]{0.4\textwidth}
        \centering
        \includegraphics[width=\textwidth]{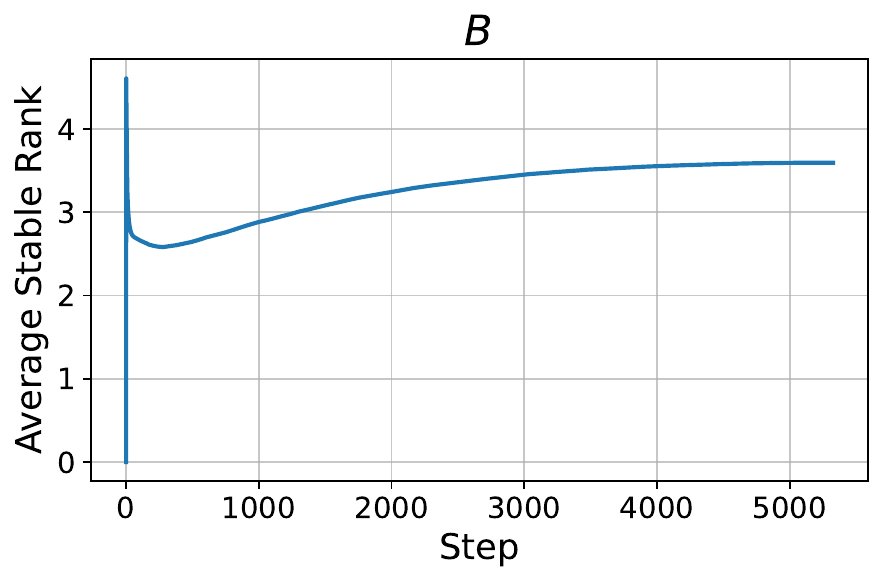}
        \vspace{-0.1em}
    \end{subfigure}
    \vspace{-0.5em}
    \caption{\textbf{Stable rank of $\mA$ and $\mB$.}
    We report the average stable rank of $\mA$ and $\mB$ across all layers.
    We fine-tune LLaMA3-8B using the Commonsense-170K dataset with rank 32.
    }
    \label{fig:stable_rank}
\end{figure}
\section{Experiments on Image Classification Tasks}\label{sec:vision}
For experiments on vision benchmarks, we use ViT-Base and ViT-Large~\citep{Dosovitskiy2021vit} pretrained on 224$\times$224 images with a 16$\times$16 patch size.
We fine-tune and evaluate them on five datasets, including Cars~\citep{cars}, CUB200~\citep{cub200}, DTD~\citep{dtd}, Food101~\citep{food101}, and SUN397~\citep{sun397}.
We use the same setup as for the natural language processing tasks, except for the epochs, learning rate, and batch size, which are detailed in Table~\ref{tab:vision_detail}.
The results in Table~\ref{tab:vision_full} show that \methodname~achieves the best average performance across all models and ranks.
In particular, we observe that \methodname~significantly reduces the performance gap between LoRA and full fine-tuning.
These results suggest that \methodname~generalizes well to vision tasks, highlighting the effectiveness of our method in various domains.

\begin{table*}
\caption{\textbf{Results on image classification tasks.}}
\label{tab:vision_full}
\centering
\resizebox{0.75\linewidth}{!}{
\begin{tabular}{ll>{\centering\arraybackslash}p{0.5cm}cccccc}
    \toprule
     Model & Method & Rank & Cars & CUB200 & DTD & Food101 & SUN397 & Avg.\\
     \toprule
     \toprule
     \multirow{13}{*}{ViT-Base} & Full FT. & - & 84.26{\scriptsize$\pm$0.19} & 86.35{\scriptsize$\pm$0.18} & 80.07{\scriptsize$\pm$0.32} & 89.27{\scriptsize$\pm$0.20} & 74.76{\scriptsize$\pm$0.26} & 82.94 \\
     \cmidrule{2-9}
     &LoRA&\multirow{6}{*}{8}&78.66{\scriptsize$\pm$0.23}&85.65{\scriptsize$\pm$0.05}&78.81{\scriptsize$\pm$0.95}&88.15{\scriptsize$\pm$0.07}&74.35{\scriptsize$\pm$0.28}&81.12\\
     &rsLoRA&&78.48{\scriptsize$\pm$0.07}&85.67{\scriptsize$\pm$0.41}&78.79{\scriptsize$\pm$0.48}&88.13{\scriptsize$\pm$0.16}&74.63{\scriptsize$\pm$0.13}&81.14\\
     &LoRA+&&79.05{\scriptsize$\pm$0.25}&85.28{\scriptsize$\pm$0.11}&78.32{\scriptsize$\pm$0.16}&88.27{\scriptsize$\pm$0.05}&72.87{\scriptsize$\pm$0.28}&80.76\\
     &PiSSA&&78.03{\scriptsize$\pm$0.32}&85.17{\scriptsize$\pm$0.24}&78.33{\scriptsize$\pm$0.46}&87.65{\scriptsize$\pm$0.07}&72.82{\scriptsize$\pm$0.10}&80.40\\
     &DoRA&&78.93{\scriptsize$\pm$0.30}&85.74{\scriptsize$\pm$0.35}&78.46{\scriptsize$\pm$0.16}&88.27{\scriptsize$\pm$0.17}&74.46{\scriptsize$\pm$0.09}&81.17\\
     \cmidrule{2-9}
     &\cellcolor{orange!25}\textbf{SDS-LoRA}&\cellcolor{orange!25}&\cellcolor{orange!25}\textbf{80.41}{\scriptsize $\pm$0.20}&\cellcolor{orange!25}\textbf{86.07}{\scriptsize $\pm$0.13}&\cellcolor{orange!25}\textbf{79.29}{\scriptsize $\pm$0.25}&\cellcolor{orange!25}\textbf{88.31}{\scriptsize $\pm$0.26}&\cellcolor{orange!25}\textbf{74.85}{\scriptsize $\pm$0.13}&\cellcolor{orange!25}\textbf{81.79}\\
     \cmidrule{2-9}
     &LoRA&\multirow{6}{*}{32}&81.48{\scriptsize$\pm$0.20}&85.75{\scriptsize$\pm$0.28}&79.34{\scriptsize$\pm$0.18}&88.37{\scriptsize$\pm$0.12}&70.82{\scriptsize$\pm$0.07}&81.15\\
     &rsLoRA&&80.57{\scriptsize$\pm$0.22}&85.42{\scriptsize$\pm$0.41}&79.50{\scriptsize$\pm$0.09}&88.81{\scriptsize$\pm$0.10}&73.92{\scriptsize$\pm$0.26}&81.64\\
     &LoRA+&&79.53{\scriptsize$\pm$0.60}&85.68{\scriptsize$\pm$0.31}&78.85{\scriptsize$\pm$0.39}&88.05{\scriptsize$\pm$0.01}&67.66{\scriptsize$\pm$0.06}&79.95\\
     &PiSSA&&80.43{\scriptsize$\pm$0.17}&84.90{\scriptsize$\pm$0.28}&78.90{\scriptsize$\pm$0.35}&87.91{\scriptsize$\pm$0.08}&69.73{\scriptsize$\pm$0.12}&80.37\\
     &DoRA&&81.07{\scriptsize$\pm$0.03}&85.67{\scriptsize$\pm$0.22}&79.56{\scriptsize$\pm$0.43}&88.83{\scriptsize$\pm$0.08}&74.05{\scriptsize$\pm$0.21}&81.84\\
     \cmidrule{2-9}
     &\cellcolor{orange!25}\textbf{SDS-LoRA}&\cellcolor{orange!25}&\cellcolor{orange!25}\textbf{82.95}{\scriptsize $\pm$0.25}&\cellcolor{orange!25}\textbf{85.99}{\scriptsize $\pm$0.22}&\cellcolor{orange!25}\textbf{80.10}{\scriptsize $\pm$0.11}&\cellcolor{orange!25}\textbf{89.02}{\scriptsize $\pm$0.16}&\cellcolor{orange!25}\textbf{74.89}{\scriptsize $\pm$0.14}&\cellcolor{orange!25}\textbf{82.59}\\
     \midrule
     \multirow{13}{*}{ViT-Large} & Full FT. & - & 88.19{\scriptsize$\pm$0.13} & 88.23{\scriptsize$\pm$0.19} & 81.71{\scriptsize$\pm$0.09} & 90.90{\scriptsize$\pm$0.17} & 76.20{\scriptsize$\pm$0.22} & 85.05 \\
     \cmidrule{2-9}
     &LoRA&\multirow{6}{*}{8}&85.07{\scriptsize$\pm$0.19}&87.53{\scriptsize$\pm$0.20}&80.64{\scriptsize$\pm$0.23}&89.96{\scriptsize$\pm$0.08}&76.30{\scriptsize$\pm$0.11}&83.90\\
     &rsLoRA&&84.61{\scriptsize$\pm$0.11}&87.75{\scriptsize$\pm$0.19}&80.87{\scriptsize$\pm$0.32}&89.92{\scriptsize$\pm$0.05}&76.45{\scriptsize$\pm$0.16}&83.92\\
     &LoRA+&&81.99{\scriptsize$\pm$0.97}&87.66{\scriptsize$\pm$0.05}&80.62{\scriptsize$\pm$0.55}&\textbf{90.06}{\scriptsize$\pm$0.03}&75.08{\scriptsize$\pm$0.14}&83.08\\
     &PiSSA&&85.59{\scriptsize$\pm$0.17}&87.55{\scriptsize$\pm$0.19}&80.34{\scriptsize$\pm$0.13}&89.60{\scriptsize$\pm$0.17}&75.41{\scriptsize$\pm$0.20}&83.70\\
     &DoRA&&84.67{\scriptsize$\pm$0.08}&87.58{\scriptsize$\pm$0.11}&80.73{\scriptsize$\pm$0.33}&90.03{\scriptsize$\pm$0.04}&76.50{\scriptsize$\pm$0.08}&83.90\\
     \cmidrule{2-9}
     &\cellcolor{orange!25}\textbf{SDS-LoRA}&\cellcolor{orange!25}&\cellcolor{orange!25}\textbf{85.73}{\scriptsize $\pm$0.11}&\cellcolor{orange!25}\textbf{87.88}{\scriptsize $\pm$0.22}&\cellcolor{orange!25}\textbf{81.00}{\scriptsize $\pm$0.13}&\cellcolor{orange!25}89.95{\scriptsize $\pm$0.16}&\cellcolor{orange!25}\textbf{76.58}{\scriptsize $\pm$0.18}&\cellcolor{orange!25}\textbf{84.23}\\
     \cmidrule{2-9}
     &LoRA&\multirow{6}{*}{32}&86.23{\scriptsize$\pm$0.05}&87.94{\scriptsize$\pm$0.22}&81.15{\scriptsize$\pm$0.32}&90.62{\scriptsize$\pm$0.03}&74.59{\scriptsize$\pm$0.25}&84.11\\
     &rsLoRA&&85.93{\scriptsize$\pm$0.30}&87.64{\scriptsize$\pm$0.06}&\textbf{81.56}{\scriptsize$\pm$0.07}&90.58{\scriptsize$\pm$0.13}&76.22{\scriptsize$\pm$0.13}&84.39\\
     &LoRA+&&82.40{\scriptsize$\pm$1.84}&87.92{\scriptsize$\pm$0.22}&80.92{\scriptsize$\pm$0.18}&90.48{\scriptsize$\pm$0.07}&75.19{\scriptsize$\pm$1.23}&83.38\\
     &PiSSA&&86.70{\scriptsize$\pm$0.15}&87.65{\scriptsize$\pm$0.15}&80.82{\scriptsize$\pm$0.26}&90.04{\scriptsize$\pm$0.09}&73.84{\scriptsize$\pm$0.08}&83.81\\
     &DoRA&&86.10{\scriptsize$\pm$0.36}&87.69{\scriptsize$\pm$0.11}&81.31{\scriptsize$\pm$0.13}&90.64{\scriptsize$\pm$0.04}&76.23{\scriptsize$\pm$0.20}&84.39\\
     \cmidrule{2-9}
     &\cellcolor{orange!25}\textbf{SDS-LoRA}&\cellcolor{orange!25}&\cellcolor{orange!25}\textbf{87.00}{\scriptsize $\pm$0.25}&\cellcolor{orange!25}\textbf{88.20}{\scriptsize $\pm$0.23}&\cellcolor{orange!25}81.35{\scriptsize $\pm$0.10}&\cellcolor{orange!25}\textbf{90.70}{\scriptsize $\pm$0.11}&\cellcolor{orange!25}\textbf{76.46}{\scriptsize $\pm$0.19}&\cellcolor{orange!25}\textbf{84.74}\\
     \bottomrule
\end{tabular}}
\end{table*}

\begin{table*}[t!]
    \caption{\textbf{Implementation details on image classification tasks.}}
    \vspace{-0.5em}
    \label{tab:vision_detail}
    \centering
    \resizebox{0.8\linewidth}{!}{
    \begin{tabular}{lcccccccccc}
        \toprule
         &\multicolumn{5}{c}{ViT-Base}& \multicolumn{5}{c}{ViT-Large} \\
         \cmidrule(lr){2-6} \cmidrule(lr){7-11}
         & Cars & CUB200 & DTD & Food101 & SUN397 & Cars & CUB200 & DTD & Food101 & SUN397 \\
         \midrule
         Epochs& 5&7 &7 &7 &5 &5 &7 &7 &7 &5\\
         Learning Rate&5e-3 &2e-3 &2e-3 &2e-3 &5e-3 &2.5e-3 &1e-3 &1e-3 &1e-3 &2.5e-3\\
         Batch Size& \multicolumn{10}{c}{64}\\
         Target Modules & \multicolumn{10}{c}{`query', `value'}\\
         \toprule
    \end{tabular}}
    \vspace{-0.5em}
\end{table*}

\section{Comparison to Preconditioning-Based Methods}\label{sec:precond_discussion}
Several previous works~\citep{riemannian2024zhang,wang2025lorapro,altlora2025yu,tastan2026loft} attempt to minimize the gap between the full fine-tuning gradient $\mG$ and its low-rank approximation $\Tilde{\mG}$ described in Equation~\ref{eq:eff_grad}.
Specifically, they modify the original gradient of the LoRA matrices as follows: $\Tilde{\nabla}_\mA \gL = s(\mB^\top\mB)^{-1}\mB^\top \mG$ and $\Tilde{\nabla}_\mB \gL = s\mG \mA^\top(\mA\mA^\top)^{-1}$ where $(\mB^\top\mB)^{-1}$ and $(\mA\mA^\top)^{-1}$ act as preconditioners.
Consequently, the effective gradient of $\mW_\text{eff}$ becomes $\Tilde{\mG} = s^2(\mG \mP_{\gR(\mA)} + \mP_{\gC(\mB)}\mG)$, attaining the worst-case-optimal form characterized in Theorem~\ref{theorem:approx}, up to a constant factor that is absorbed into the learning rate.

Although these preconditioning-based methods also achieve optimal gradient approximation, our work differs from theirs in several key aspects.
First, our work focuses on addressing the adverse effects of anisotropic scaling induced by singular values, whereas prior works primarily aim to improve the quality of gradient approximation.
While both approaches yield the same update rule in the effective weight space $\mW_\text{eff}$, preconditioning-based methods fail to eliminate the influence of singular values in the gradients of the LoRA matrices.
Substituting the SVD of the LoRA matrices into $\Tilde{\nabla}_\mA \gL$, we obtain $\Tilde{\nabla}_\mA \gL = s \mV_\mB \boldsymbol{\Sigma}^{-1}_\mB \mU_\mB^\top \mG$, indicating that these methods do not address the adverse effects of anisotropic scaling induced by singular values.
Second, because preconditioning-based methods modify the original gradients, they do not guarantee the steepest descent direction in the parameter space. The descent lemma suggests that deviations from the true gradient may lead to a weaker bound on loss reduction.
Indeed, Wang et al.~\citep{wang2025lorapro} show that performance degrades as the modified gradients diverge from the original gradients.
The results in Figure~\ref{fig:loss_curve} demonstrate that SDS-LoRA achieves superior loss convergence compared to preconditioning-based methods.
Third, modifying the original gradients requires explicitly tracking optimizer momentum (e.g., in Adam).
These methods introduce additional mechanisms to adjust the momentum based on the modified gradients, which may hinder practical adoption.
In contrast, \methodname~structurally decouples the influence of singular values from the gradients without modifying the original gradient, allowing straightforward adoption.

\paragraph{Optimality in the effective weight space does not determine the training trajectory.}
Theorem~\ref{theorem:approx} characterizes the optimal effective gradient $\Tilde{\mG}$, that is, the update applied to $\mW_\text{eff}$ at a single step given the current $\mA$ and $\mB$.
Preconditioning-based methods attain the same optimum, and are therefore indistinguishable from \methodname~by this criterion alone.
This criterion, however, does not pin down the optimization trajectory.
At the optimum, $\Tilde{\mG} = s^2(\mG \mP_{\gR(\mA)} + \mP_{\gC(\mB)}\mG)$ is determined entirely by the two subspaces $\gR(\mA)$ and $\gC(\mB)$, and how much of $\mG$ these subspaces capture is exactly the quantity $\alpha$ that governs the convergence rate in Theorem~\ref{theorem:convergence}.
These subspaces are in turn produced by the updates applied to $\mA$ and $\mB$, which differ across methods: preconditioning replaces the gradient of $\mA$ with $\Tilde{\nabla}_\mA \gL = s\mV_\mB \boldsymbol{\Sigma}_\mB^{-1} \mU_\mB^\top \mG$, which remains anisotropic and merely inverts the direction of the distortion, whereas \methodname~yields the isometric $s\mQ_\mB^\top \mG$.
Two methods that agree on the one-step optimum in the effective weight space can thus drive $\gR(\mA)$ and $\gC(\mB)$ along very different paths.
Figure~\ref{fig:cossim} makes this concrete: although all the compared preconditioning-based methods attain the same optimal form of $\Tilde{\mG}$, \methodname~sustains a markedly higher alignment between $\mG$ and $\Tilde{\mG}$ throughout training, a gap that we attribute to the difference in the subspaces the methods reach.
The `\methodname~w/o detachment' variant in Table~\ref{tab:orthogonal_lora} isolates the same effect within our own formulation.
In short, the two questions are distinct: Theorem~\ref{theorem:approx} states which $\Tilde{\mG}$ one should aim for, whereas \methodname~additionally governs the quality of the gradient that $\mA$ and $\mB$ actually receive on the way there.

\paragraph{Comparison to LoRA-RITE and RefLoRA.}
LoRA-RITE~\citep{yen2025lora} and RefLoRA~\citep{zhang2025reflora} address a closely related problem: the update of one low-rank factor is scaled by the magnitude of the other, which makes training sensitive to the particular factorization chosen.\footnote{Both works write $\Delta \mW = \mA\mB^\top$, so their $\mA$ plays the role of our $\mB$. We follow their own notation in this paragraph.}
RefLoRA observes that many pairs $(\mA_t, \mB_t)$ represent the same $\Delta \mW$ and, at every step, selects the best balanced pair among them.
It computes a matrix $\Tilde{\mS}_t$ from $(\mA_t^\top \mA_t)^{-1}$ and $\mB_t^\top \mB_t$, and preconditions the gradients fed to AdamW as $\mG \mB_t \Tilde{\mS}_t^{-1}$ for $\mA$ and $\mG^\top \mA_t \Tilde{\mS}_t$ for $\mB$, so that the resulting update to $\Delta \mW$ is the same regardless of which valid pair one starts from.
The gradient with respect to $\mA$, however, still carries $\mB_t$ and $\Tilde{\mS}_t$: in the balanced state $\mA_t^\top \mA_t = \mB_t^\top \mB_t$, which is exactly the state RefLoRA steers toward, $\Tilde{\mS}_t = \mI$ and the preconditioned gradient reduces to LoRA's own $\mG \mB_t$, leaving anisotropic gradient scaling untouched.
LoRA-RITE first builds a scale-free unmagnified gradient $\bar{\nabla}_\mA = \nabla_\mZ \mU_\mB$, where $\mU_\mB$ is obtained from the polar decomposition $\mB = \mU_\mB \mR_\mB$, so that the magnitude of $\mB$ is carried by $\mR_\mB$ and dropped.
The applied update, however, is $\delta \mA_t = -\eta \bar{\mM}_{\mA_t} \mR_\mB^{-\top}$, which reintroduces the magnitude of $\mB$ through $\mR_\mB^{-\top}$.
Since $\mR_\mB^{-\top} \mB^\top = \mU_\mB^\top$, this reintroduced magnitude cancels once $\delta \mA_t$ is folded into $\Delta \mW = \delta \mA_t \mB^\top$, but the update applied to $\mA$ itself still depends on the magnitude of $\mB$, only inverted.
Table~\ref{tab:precond_mechanism} summarizes these mechanism-level differences, and Table~\ref{tab:reflora_comparison} compares the fine-tuning performance.
\methodname~achieves the highest average accuracy on both models, supporting our claim that it is important to resolve anisotropic gradient scaling in the update of $\mA$ and $\mB$, not only in the update of $\Delta \mW$.

\begin{table}[t!]
\caption{\textbf{Mechanism-level comparison to gradient-based methods.}
`Invariant $\Delta \mW$ update' indicates whether the update of $\Delta \mW$ is independent of the scaling of $\mA$ and $\mB$, `AGS in $\mA, \mB$ resolved' indicates whether anisotropic gradient scaling is removed from the updates of $\mA$ and $\mB$, and `Original gradient' indicates whether the original gradients of $\mA$ and $\mB$ are used without modification.
Among these methods, LoRA-Pro, ScaledAdamW, AltLoRA, and LoRA-RITE additionally attain the optimal form $\Tilde{\mG} = s^2(\mG \mP_{\gR(\mA)} + \mP_{\gC(\mB)}\mG)$, and LoRA-GA attains it at the first training step only, whereas RefLoRA targets invariance to the choice of factorization rather than the gradient approximation quality itself.
}
\label{tab:precond_mechanism}
\centering
\resizebox{0.85\linewidth}{!}{
\begin{tabular}{lccc}
     \toprule
     Method & Scale-Invariant $\Delta \mW$ update & AGS in $\mA, \mB$ resolved & Original gradient\\
     \toprule
     \toprule
     LoRA-GA & \textcolor{red}{\ding{56}} & \textcolor{red}{\ding{56}} & \textcolor{green}{\ding{52}}\\
     LoRA-Pro / ScaledAdamW / AltLoRA & \textcolor{green}{\ding{52}} & \textcolor{red}{\ding{56}} & \textcolor{red}{\ding{56}}\\
     RefLoRA & \textcolor{green}{\ding{52}} & \textcolor{red}{\ding{56}} & \textcolor{red}{\ding{56}}\\
     LoRA-RITE & \textcolor{green}{\ding{52}} & \textcolor{red}{\ding{56}} & \textcolor{red}{\ding{56}}\\
     \midrule
     \cellcolor{orange!25}\textbf{\methodname} & \cellcolor{orange!25}\textcolor{green}{\ding{52}} & \cellcolor{orange!25}\textcolor{green}{\ding{52}} & \cellcolor{orange!25}\textcolor{green}{\ding{52}}\\
     \bottomrule
\end{tabular}}
\vspace{-0.5em}
\end{table}

\begin{table}[t!]
\caption{\textbf{Comparison to LoRA-RITE and RefLoRA.}
We report the average accuracy on the commonsense reasoning benchmarks with rank 32 from a single run.
The reported numbers are taken from RefLoRA~\citep{zhang2025reflora}, whose experimental setup may differ from ours.
In our reproduction, all aspects of the experimental setup follow ours except for the learning rate, which is searched over $\{$1e-4, 2e-4, 5e-4, 1e-3$\}$.
}
\label{tab:reflora_comparison}
\centering
\resizebox{0.55\linewidth}{!}{
\begin{tabular}{lcc}
     \toprule
     Method & Gemma-2B & LLaMA3-8B\\
     \toprule
     \toprule
     LoRA-RITE (reported) & - & 85.56\\
     RefLoRA (reported) & - & 85.75\\
     \midrule
     LoRA-RITE (reproduced) & 72.75 & 85.49\\
     RefLoRA (reproduced) & 73.97 & 85.83\\
     \midrule
     \cellcolor{orange!25}\textbf{\methodname} & \cellcolor{orange!25}\textbf{74.80} & \cellcolor{orange!25}\textbf{86.08}\\
     \bottomrule
\end{tabular}}
\vspace{-0.5em}
\end{table}

\section{Comparison to Orthogonalization-Based LoRA Variants}\label{sec:orthogonal_lora}
Several LoRA variants also employ orthonormal bases obtained by QR decomposition.
Our contribution, however, is not the use of orthonormal bases per se, but the structural decoupling of the singular values from the backward pass.
We clarify this distinction at the mechanism level and validate it empirically.

\paragraph{OLoRA.}
OLoRA~\citep{buyukakyuz2024olora} applies QR decomposition once, at initialization, to the pretrained weight $\mW_0 = \mQ\mR$, sets $\mB$ to the first $r$ columns of $\mQ$ and $\mA$ to the first $r$ rows of $\mR$, and then trains the standard LoRA formulation.
Its backward pass is therefore identical to that of LoRA and still carries $\boldsymbol{\Sigma}_\mB$.
Since $\mB$ is initialized with orthonormal columns, $\boldsymbol{\Sigma}\mB = \mI_r$ holds at the first iteration, but nothing constrains $\boldsymbol{\Sigma}\mB$ once training proceeds.
OLoRA thus addresses anisotropic gradient scaling only at initialization.

\paragraph{QR-LoRA.}
QR-LoRA~\citep{yang2025qrlora} applies QR decomposition to the transpose of a rank-$r$ core of the pretrained weight, $\mW_\text{core}^\top = \mQ\mR$, and parameterizes the weight update as $\Delta \mW = (\mQ \Delta \mR)^\top$, where the orthonormal factor $\mQ$ is frozen and only the increment $\Delta \mR$ to the triangular factor is trained.
Because the gradient reaches $\Delta \mR$ only through the orthonormal $\mQ$, anisotropic gradient scaling is indeed removed.
However, the fixed $\mQ$ confines the row space of $\Delta \mW$ to a subspace determined by the pretrained weight, so QR-LoRA cannot represent an arbitrary rank-$r$ update.

\paragraph{Comparison.}
As summarized in Table~\ref{tab:orthogonal_lora}, neither method satisfies both desirable properties at once, whereas \methodname~removes anisotropic gradient scaling from the backward pass without restricting the representational capacity of the forward pass.
To further isolate the source of the gain, we evaluate two additional variants: QR-LoRA with a trainable $\mQ$, which tests whether its fixed $\mQ$ causes the collapse of QR-LoRA (4.67\% in Table~\ref{tab:orthogonal_lora}), and \methodname~without detachment, which allows gradients to flow through the QR decomposition and thereby reintroduces anisotropic gradient scaling.
\methodname~and \methodname~without detachment differ only in whether $\mQ_\mA$ and $\mQ_\mB$ are detached during the backward pass, and the gap between them indicates that the gain comes from removing anisotropic gradient scaling rather than from the use of orthonormal bases.

\begin{table}[t!]
\caption{\textbf{Comparison to orthogonalization-based LoRA variants.}
`Removes AGS' indicates whether anisotropic gradient scaling is removed throughout training, and `Full RC of $\Delta \mW$' indicates the full representational capacity of $\Delta \mW$.
We report the average accuracy on the commonsense reasoning benchmarks with Gemma-2B at rank 32 from a single run.
}
\label{tab:orthogonal_lora}
\centering
\resizebox{0.7\linewidth}{!}{
\begin{tabular}{lccc}
     \toprule
     Method & Removes AGS & Full RC of $\Delta \mW$ & Avg. Accuracy\\
     \toprule
     \toprule
     OLoRA & \textcolor{red}{\ding{56}} (only at init.) & \textcolor{green}{\ding{52}} & 71.24\\
     QR-LoRA & \textcolor{green}{\ding{52}} & \textcolor{red}{\ding{56}} & 4.67\\
     QR-LoRA (trainable $\mQ$) & \textcolor{red}{\ding{56}} & \textcolor{green}{\ding{52}} & 69.88\\
     \methodname~w/o detachment & \textcolor{red}{\ding{56}} & \textcolor{green}{\ding{52}} & 72.86\\
     \midrule
     \cellcolor{orange!25}\textbf{\methodname} & \cellcolor{orange!25}\textcolor{green}{\ding{52}} & \cellcolor{orange!25}\textcolor{green}{\ding{52}} & \cellcolor{orange!25}\textbf{74.80}\\
     \bottomrule
\end{tabular}}
\vspace{-0.5em}
\end{table}

\begin{table*}[t!]
    \caption{\textbf{Additional implementation details.}}
    \vspace{-0.5em}
    \label{tab:additional_detail}
    \centering
    \resizebox{0.7\linewidth}{!}{
    \begin{tabular}{lcccc}
        \toprule
        & \multicolumn{2}{c}{Full FT.} & \multicolumn{2}{c}{LoRA-based}\\
        \cmidrule(lr){2-3} \cmidrule(lr){4-5}
        & Gemma-2B & LLaMA3-8B & Gemma-2B & LLaMA3-8B\\
         \midrule
         Learning Rate&2e-5&1e-5&2e-4&1e-4\\
         Learning Rate Scheduler& \multicolumn{4}{c}{cosine scheduler}\\
         Epochs& \multicolumn{4}{c}{1}\\
         Batch Size&\multicolumn{4}{c}{32}\\
         Target Modules & \multicolumn{4}{c}{`q\_proj', `k\_proj', `v\_proj', `up\_proj', `down\_proj', `o\_proj', `gate\_proj'}\\
         \toprule
    \end{tabular}}
    \vspace{-0.5em}
\end{table*}

\section*{Broader Impacts}
This research enhances the performance of LoRA, enabling low-rank adaptation to achieve results comparable to full fine-tuning.
By improving the efficacy of parameter-efficient fine-tuning, our work allows high-quality model customization even in resource-constrained settings.
This contributes to the development of more accurate and reliable AI systems across various domains.
While improved performance could be utilized for unintended purposes, its primary impact lies in fostering robust and high-performing open-source models that do not require massive computational overhead.

\newpage
\section*{NeurIPS Paper Checklist}

\begin{enumerate}

\item {\bf Claims}
    \item[] Question: Do the main claims made in the abstract and introduction accurately reflect the paper's contributions and scope?
    \item[] Answer: \answerYes{} 
    \item[] Justification: Throughout the paper, we consistently claim that LoRA suffers from anisotropic gradient scaling (Figure~\ref{fig:motivation} and Theorem~\ref{theorem:approx}) and our method addresses this successfully (Theorem~\ref{theorem:convergence}).

\item {\bf Limitations}
    \item[] Question: Does the paper discuss the limitations of the work performed by the authors?
    \item[] Answer: \answerYes{} 
    \item[] Justification: We discuss the limitations of our work in the ``Limitations'' paragraph following the conclusion.

\item {\bf Theory assumptions and proofs}
    \item[] Question: For each theoretical result, does the paper provide the full set of assumptions and a complete (and correct) proof?
    \item[] Answer: \answerYes{} 
    \item[] Justification: We explicitly state the assumptions in Theorem~\ref{theorem:approx} and Theorem~\ref{theorem:convergence}, and provide the complete proof in Section~\ref{sec:approx_proof} and Section~\ref{sec:convergence_proof}.

    \item {\bf Experimental result reproducibility}
    \item[] Question: Does the paper fully disclose all the information needed to reproduce the main experimental results of the paper to the extent that it affects the main claims and/or conclusions of the paper (regardless of whether the code and data are provided or not)?
    \item[] Answer: \answerYes{} 
    \item[] Justification: We provide a detailed algorithm for the proposed method in Algorithm~\ref{alg:training}.

\item {\bf Open access to data and code}
    \item[] Question: Does the paper provide open access to the data and code, with sufficient instructions to faithfully reproduce the main experimental results, as described in supplemental material?
    \item[] Answer: \answerYes{} 
    \item[] Justification: The code is publicly available at \url{https://github.com/JungHunOh/SDS-LoRA_NeurIPS2026}, and all datasets used in our experiments are publicly available and cited in Section~\ref{sec:experimental_setup}.

\item {\bf Experimental setting/details}
    \item[] Question: Does the paper specify all the training and test details (e.g., data splits, hyperparameters, how they were chosen, type of optimizer) necessary to understand the results?
    \item[] Answer: \answerYes{} 
    \item[] Justification: We provide experimental details in Section~\ref{sec:experimental_setup} and Table~\ref{tab:additional_detail}.

\item {\bf Experiment statistical significance}
    \item[] Question: Does the paper report error bars suitably and correctly defined or other appropriate information about the statistical significance of the experiments?
    \item[] Answer: \answerYes{} 
    \item[] Justification: For the main experimental results (Tables~\ref{tab:commonsense}, \ref{tab:nlg}, \ref{tab:vision_main}, \ref{tab:gradientbased_r8}, \ref{tab:ablation}, \ref{tab:interval_sensitivity}, and \ref{tab:vision_full}), we report the mean and standard deviation across three independent runs.
    Some additional analyses in the appendix (Tables~\ref{tab:reflora_comparison}, \ref{tab:orthogonal_lora}, \ref{tab:warmup_ablation}, and \ref{tab:q_frozen}) are reported from a single run, as noted in their captions.

\item {\bf Experiments compute resources}
    \item[] Question: For each experiment, does the paper provide sufficient information on the computer resources (type of compute workers, memory, time of execution) needed to reproduce the experiments?
    \item[] Answer: \answerYes{} 
    \item[] Justification: The specific GPU used for our experiments is described in Section~\ref{sec:experimental_setup}.
    We provide an analysis of training overhead in Table~\ref{tab:overhead}.
    
\item {\bf Code of ethics}
    \item[] Question: Does the research conducted in the paper conform, in every respect, with the NeurIPS Code of Ethics \url{https://neurips.cc/public/EthicsGuidelines}?
    \item[] Answer: \answerYes{} 
    \item[] Justification: We confirm that this work strictly adheres to the Code of Ethics.

\item {\bf Broader impacts}
    \item[] Question: Does the paper discuss both potential positive societal impacts and negative societal impacts of the work performed?
    \item[] Answer: \answerYes{} 
    \item[] Justification: We discuss the broader impacts of our work in the ``Broader Impacts'' section at the end of the appendix.
    
\item {\bf Safeguards}
    \item[] Question: Does the paper describe safeguards that have been put in place for responsible release of data or models that have a high risk for misuse (e.g., pre-trained language models, image generators, or scraped datasets)?
    \item[] Answer: \answerNA{} 
    \item[] Justification: 

\item {\bf Licenses for existing assets}
    \item[] Question: Are the creators or original owners of assets (e.g., code, data, models), used in the paper, properly credited and are the license and terms of use explicitly mentioned and properly respected?
    \item[] Answer: \answerYes{} 
    \item[] Justification: We have cited the pre-trained models and datasets used in this study.

\item {\bf New assets}
    \item[] Question: Are new assets introduced in the paper well documented and is the documentation provided alongside the assets?
    \item[] Answer: \answerNA{} 
    \item[] Justification:

\item {\bf Crowdsourcing and research with human subjects}
    \item[] Question: For crowdsourcing experiments and research with human subjects, does the paper include the full text of instructions given to participants and screenshots, if applicable, as well as details about compensation (if any)? 
    \item[] Answer: \answerNA{} 
    \item[] Justification: 

\item {\bf Institutional review board (IRB) approvals or equivalent for research with human subjects}
    \item[] Question: Does the paper describe potential risks incurred by study participants, whether such risks were disclosed to the subjects, and whether Institutional Review Board (IRB) approvals (or an equivalent approval/review based on the requirements of your country or institution) were obtained?
    \item[] Answer: \answerNA{} 
    \item[] Justification: 

\item {\bf Declaration of LLM usage}
    \item[] Question: Does the paper describe the usage of LLMs if it is an important, original, or non-standard component of the core methods in this research? Note that if the LLM is used only for writing, editing, or formatting purposes and does \emph{not} impact the core methodology, scientific rigor, or originality of the research, declaration is not required.
    \item[] Answer: \answerNA{} 
    \item[] Justification: 

\end{enumerate}

\end{document}